\RequirePackage{fix-cm}

\documentclass[smallextended]{svjour3}       
\smartqed  
\usepackage{graphicx}

\usepackage{utopia}
\usepackage{amssymb}
\usepackage{amsfonts}
\usepackage{amsmath}\allowdisplaybreaks
\usepackage{mathtools}
\usepackage{multirow}
\usepackage{enumerate}
\usepackage{bm}
\usepackage{algorithmic}
\usepackage{hyperref}
\usepackage[american]{babel}
\mathtoolsset{showonlyrefs,showmanualtags}

\newtheorem{algorithm}{Algorithm}

\newenvironment{myproof}
{\par\noindent\textit{Proof}\ \enspace\ignorespaces\begin{allowdisplaybreaks}}
{\end{allowdisplaybreaks}\hspace{\stretch{1}}$\square$}

\begin{document}

\title{Running Time Analysis of the (1+1)-EA for OneMax and LeadingOnes under Bit-wise Noise\thanks{A preliminary version of this paper has appeared at GECCO'17~\cite{qian2017noise}.}
}
\subtitle{}

\titlerunning{Running Time Analysis of the (1+1)-EA under Bit-wise Noise}        

\author{Chao Qian$^{1}$         \and
        Chao Bian$^{1}$         \and
        Wu Jiang$^{1}$         \and
        Ke Tang$^{2*}$
}



\institute{1\quad Anhui Province Key Lab of Big Data Analysis and Application, School of Computer Science and Technology, University of Science and Technology of China, Hefei 230027, China\vspace{0.3em}\\
2\quad Shenzhen Key Lab of Computational Intelligence, Department of Computer Science and Engineering, Southern University of Science and Technology, Shenzhen 518055, China\vspace{0.3em}\\
*\quad Corresponding author, tangk3@sustc.edu.cn}

\date{Received: date / Accepted: date}

\maketitle

\begin{abstract}
In many real-world optimization problems, the objective function evaluation is subject to noise, and we cannot obtain the exact objective value. Evolutionary algorithms (EAs), a type of general-purpose randomized optimization algorithm, have been shown to be able to solve noisy optimization problems well. However, previous theoretical analyses of EAs mainly focused on noise-free optimization, which makes the theoretical understanding largely insufficient for the noisy case. Meanwhile, the few existing theoretical studies under noise often considered the one-bit noise model, which flips a randomly chosen bit of a solution before evaluation; while in many realistic applications, several bits of a solution can be changed simultaneously. In this paper, we study a natural extension of one-bit noise, the bit-wise noise model, which independently flips each bit of a solution with some probability. We analyze the running time of the (1+1)-EA solving OneMax and LeadingOnes under bit-wise noise for the first time, and derive the ranges of the noise level for polynomial and super-polynomial running time bounds. The analysis on LeadingOnes under bit-wise noise can be easily transferred to one-bit noise, and improves the previously known results. Since our analysis discloses that the (1+1)-EA can be efficient only under low noise levels, we also study whether the sampling strategy can bring robustness to noise. We prove that using sampling can significantly increase the largest noise level allowing a polynomial running time, that is, sampling is robust to noise.
\end{abstract}

\section{Introduction}

In real-world optimization tasks, the exact objective (i.e., fitness) function evaluation of candidate solutions is often impossible, instead we can obtain only a noisy one due to a wide range of uncertainties~\cite{jin2005evolutionary}. For example, in machine learning, a prediction model is evaluated only on a limited amount of data, which makes the estimated performance deviate from the true performance; in product design, the design variables can be subject to perturbations due to manufacturing tolerances, which brings noise.

In the presence of noise, the difficulty of solving an optimization problem may increase. Evolutionary algorithms (EAs)~\cite{back:96}, inspired by natural phenomena, are a type of randomized metaheuristic optimization algorithm. They are likely to be able to handle noise, since the corresponding natural phenomena have been well processed in noisy natural environments. In fact, EAs have been successfully applied to solve many noisy optimization problems~\cite{chang2006new,ma2006evolutionary}.

Compared with the application, the theoretical analysis of EAs is far behind. But in the last two decades, much effort has been devoted to the running time analysis (an essential theoretical aspect) of EAs. Numerous analytical results for EAs solving synthetic problems as well as combinatorial problems have been derived, e.g.,~\cite{auger2011theory,neumann2010bioinspired}. Meanwhile, a few general approaches for running time analysis have been proposed, e.g., drift analysis~\cite{doerr:goldberg:11,doerr:etal:GECCO10,he2001drift}, fitness-level methods~\cite{corus2014level,sudholt2011general}, and switch analysis~\cite{yu2014switch}.

However, previous running time analyses of EAs mainly focused on noise-free optimization, where the fitness evaluation is exact. Only a few pieces of work on noisy evolutionary optimization have been reported, which mainly considered two kinds of noise models, prior and posterior. The prior noise comes from the variation on a solution, e.g., one-bit noise~\cite{droste2004analysis} flips a random bit of a binary solution before evaluation with probability $p$. The posterior noise comes from the variation on the fitness of a solution, e.g., additive noise~\cite{giessen2014robustness} adds a value randomly drawn from some distribution. Droste~\cite{droste2004analysis} first analyzed the (1+1)-EA on the OneMax problem in the presence of one-bit noise and showed that the tight range of the noise probability $p$ allowing a polynomial running time is $O(\log n/n)$, where $n$ is the problem size. Gie{\ss}en and K{\"o}tzing~\cite{giessen2014robustness} recently studied the LeadingOnes problem, and proved that the expected running time is polynomial if $p \leq 1/(6en^2)$ and exponential if $p=1/2$. They also considered additive noise with variance $\sigma^2$, and proved that the (1+1)-EA can solve OneMax and LeadingOnes in polynomial time when $\sigma^2=O(\log n/n)$ and $\sigma^2\leq 1/(12en^2)$, respectively.

For inefficient optimization of the (1+1)-EA under high noise levels, some implicit mechanisms of EAs were proved to be robust to noise. In~\cite{giessen2014robustness}, it was shown that the ($\mu$+1)-EA with a small population of size $\Theta(\log n)$ can solve OneMax in polynomial time even if the probability of one-bit noise reaches 1. The robustness of populations to noise was also proved in the setting of non-elitist EAs~\cite{dang2015efficient,prugel2015run}. However, Friedrich et al.~\cite{friedrich2015benefit} showed the limitation of populations by proving that the ($\mu$+1)-EA needs super-polynomial time for solving OneMax under additive Gaussian noise $\mathcal{N}(0,\sigma^2)$ with $\sigma^2 \geq n^3$. This difficulty can be overcome by the compact genetic algorithm (cGA)~\cite{friedrich2015benefit} and a simple Ant Colony Optimization (ACO) algorithm~\cite{friedrich2015robustness}, both of which find the optimal solution in polynomial time with a high probability. ACO was also shown to be able to efficiently find solutions with reasonable approximations on some instances of the single-destination shortest path problem with edge weights disturbed by noise~\cite{doerr2012ants,feldmann2013optimizing,sudholt2012simple}.

The ability of explicit noise handling strategies was also theoretically studied. Qian et al.~\cite{qian2015noise} proved that the threshold selection strategy is robust to noise: the expected running time of the (1+1)-EA using threshold selection on OneMax under one-bit noise is always polynomial regardless of the noise probability $p$. For the (1+1)-EA solving OneMax under one-bit noise with $p=1$ or additive Gaussian noise $\mathcal{N}(0,\sigma^2)$ with $\sigma^2 \geq 1$, the sampling strategy was shown to be able to reduce the running time from exponential to polynomial~\cite{qian2016sampling}. The robustness of sampling to noise was also proved for the (1+1)-EA solving LeadingOnes under one-bit noise with $p=1/2$ or additive Gaussian noise with $\sigma^2 \geq n^2$. Akimoto et al.~\cite{akimoto2015analysis} proved that sampling with a large sample size can make optimization under additive unbiased noise behave as optimization in a noise-free environment. The interplay between sampling and implicit noise-handling mechanisms (e.g., crossover) has been statistically studied in~\cite{friedrich2017resampling}.

The noise models considered in the studies mentioned above are summarized in Table~\ref{table-noise}. We can observe that for the prior noise model, one-bit noise was mainly considered, which flips a random bit of a solution before evaluation with probability $p$. However, the noise model, which can change several bits of a solution simultaneously, may be more realistic and needs to be studied, as mentioned in the first noisy theoretical work~\cite{droste2004analysis}.

\begin{table*}\centering
\caption{The noise models mainly considered in running time analyses on noisy evolutionary optimization.}\label{table-noise}
\begin{tabular}{ll}
\hline
& \\[-8pt]
\!\!\!References & Noise models \!\!\!\\[2pt]
\hline
&   \\[-8pt]
\!\!\!Droste~\cite{droste2004analysis}, Qian et al.~\cite{qian2015noise} & one-bit noise \!\!\!\\[2pt]
\!\!\!Akimoto et al.~\cite{akimoto2015analysis} & additive noise\!\!\!\\[2pt]
\!\!\!Prugel-Bennett et al.~\cite{prugel2015run}, Friedrich et al.~\cite{friedrich2015robustness,friedrich2015benefit,friedrich2017resampling} & additive Gaussian noise\!\!\!\\[2pt]
\!\!\!Dang and Lehre~\cite{dang2015efficient}, Gie{\ss}en and K{\"o}tzing~\cite{giessen2014robustness} & one-bit noise, additive noise\!\!\!\\[2pt]
\!\!\!Qian et al.~\cite{qian2016sampling} & one-bit noise, additive Gaussian noise\!\!\!\\[2pt]
\!\!\!Doerr et al.~\cite{doerr2012ants}, Feldmann and K{\"o}tzing~\cite{feldmann2013optimizing}, & the single-destination shortest path \\
\!\!\!Sudholt and Thyssen~\cite{sudholt2012simple} & problem with stochastic edge weights\!\!\!\\[2pt]
\hline
& \\[-8pt]
\!\!\!This paper & bit-wise noise \!\!\!\\[2pt]
\hline
\end{tabular}\vspace{-1em}
\end{table*}

In this paper, we study the bit-wise noise model, which is characterized by a pair $(p,q)$ of parameters. It happens with probability $p$, and independently flips each bit of a solution with probability $q$ before evaluation. We analyze the running time of the (1+1)-EA solving OneMax and LeadingOnes under bit-wise noise with two specific parameter settings $(p,\frac{1}{n})$ and $(1,q)$. The ranges of $p$ and $q$ for a polynomial upper bound and a super-polynomial lower bound are derived, as shown in the middle row of Table~\ref{table_runtime}. For the (1+1)-EA on LeadingOnes, we also transfer the running time bounds from bit-wise noise $(p,\frac{1}{n})$ to one-bit noise by using the same proof procedure. As shown in the bottom right of Table~\ref{table_runtime}, our results improve the previously known ones~\cite{giessen2014robustness}.

Note that for the (1+1)-EA on LeadingOnes, the current analysis (as shown in the last column of Table~\ref{table_runtime}) does not cover all the ranges of $p$ and $q$. We thus conduct experiments to estimate the expected running time for the uncovered values of $p$ and $q$. The empirical results show that the currently derived ranges of $p$ and $q$ allowing a polynomial running time are possibly tight.

\begin{table*}\centering
\caption{For the running time of the (1+1)-EA on OneMax and LeadingOnes under prior noise models, the ranges of noise parameters for a polynomial upper bound and a super-polynomial lower bound are shown below.}\label{table_runtime}
\begin{tabular}{l|l|l}
\hline
& & \\[-8pt]
(1+1)-EA & OneMax & LeadingOnes  \\[2pt]
\hline
& &  \\[-8pt]
bit-wise noise $(p,\frac{1}{n})$ & $O(\log n /n)$, $\omega(\log n/n)$  &  $O(\log n /n^2)$, $\omega(\log n/n)$ \\[2pt]
bit-wise noise $(1,q)$ & $O(\log n /n^2)$, $\omega(\log n/n^2)$~\cite{giessen2014robustness}   & $O(\log n /n^3)$, $\omega(\log n/n^2)$
\\[2pt]
\hline
& &  \\[-8pt]
\multirow{2}{*}{one-bit noise} & \multirow{2}{*}{$O(\log n /n)$, $\omega(\log n/n)$~\cite{droste2004analysis}} & $ [0,1/(6en^2)], 1/2$~\cite{giessen2014robustness};\\
& & $O(\log n /n^2)$, $\omega(\log n/n)$\\[2pt]
\hline
\end{tabular}\vspace{-1em}
\end{table*}

\begin{table*}\centering
\caption{For the running time of the (1+1)-EA using sampling on OneMax and LeadingOnes under prior noise models, the ranges of noise parameters for a polynomial upper bound and a super-polynomial lower bound are shown below.}\label{table_runtime_sampling}
\begin{tabular}{l|l|l}
\hline
& & \\[-8pt]
\!\!(1+1)-EA using sampling\! & \!\!OneMax \!&\!\! LeadingOnes \!\! \\[2pt]
\hline
& &  \\[-8pt]
\!\!bit-wise noise $(p,\frac{1}{n})$ \!&\!\! $[0,1]$, $\emptyset$  \!&\!\!  $[0,1]$, $\emptyset$\!\!\\[2pt]
\!\!bit-wise noise $(1,q)$ \!&\!\! $1/2\!-\!1/n^{O(1)}$, $1/2\!-\!1/n^{\omega(1)} \cup [1/2,1]$   \!&\!\! $O(\log n /n)$, $\omega(\log n/n)\!\!$
\\[2pt]
\hline
& &  \\[-8pt]
\!\!one-bit noise \!&\!\! $[0,1]$, $\emptyset$ \!&\!\! $[0,1]$, $\emptyset$\!\!\\[2pt]
\hline
\end{tabular}\vspace{-1em}
\end{table*}

From the results in Table~\ref{table_runtime}, we find that the (1+1)-EA is efficient only under low noise levels. For example, for the (1+1)-EA solving OneMax under bit-wise noise $(p,\frac{1}{n})$, the expected running time is polynomial only when $p=O(\log n/n)$. We then study whether the sampling strategy can bring robustness to noise. Sampling is a popular way to cope with noise in fitness evaluation~\cite{arnold2006general}, which, instead of evaluating the fitness of one solution only once, evaluates the fitness multiple ($m$) times and then uses the average to approximate the true fitness. We analyze the running time of the (1+1)-EA using sampling under both bit-wise noise and one-bit noise. The ranges of $p$ and $q$ for a polynomial upper bound and a super-polynomial lower bound are shown in Table~\ref{table_runtime_sampling}. Our analysis covers all the ranges of $p$ and $q$. Note that for proving a polynomial upper bound, it is sufficient to show that using sampling with a specific sample size $m$ can guarantee a polynomial running time, while for proving a super-polynomial lower bound, we need to prove that using sampling with any polynomially bounded $m$ fails to guarantee a polynomial running time. Compared with the results in Table~\ref{table_runtime}, we find that using sampling significantly improves the noise-tolerance ability. For example, by using sampling with $m=4n^3$, the (1+1)-EA now can always solve OneMax under bit-wise noise $(p,\frac{1}{n})$ in polynomial time.

From the analysis, we also find the reason why sampling is effective or not. Let $f(x)$ and $f^n(x)$ denote the true and noisy fitness of a solution, respectively. For two solutions $x$ and $y$ with $f(x)>f(y)$, when the noise level is high (i.e., the values of $p$ and $q$ are large), the probability $\mathrm{P}(f^n(x) \leq f^n(y))$ (i.e., the true worse solution $y$ appears to be better) becomes large, which will mislead the search direction and then lead to a super-polynomial running time. In such a situation, if the expected gap between $f^n(x)$ and $f^n(y)$ is positive, sampling will increase this trend and make $\mathrm{P}(f^n(x) \leq f^n(y))$ sufficiently small; if it is negative (e.g., on OneMax under bit-wise noise $(1,q)$ with $q\geq 1/2$), sampling will continue to increase $\mathrm{P}(f^n(x) \leq f^n(y))$, and obviously will not work. We also note that if the positive gap between $f^n(x)$ and $f^n(y)$ is too small (e.g., on OneMax under bit-wise noise $(1,q)$ with $q= 1/2-1/n^{\omega(1)}$), a polynomial sample size will be not sufficient and sampling also fails to guarantee a polynomial running time.

This paper extends our preliminary work~\cite{qian2017noise}. Since the theoretical analysis on the LeadingOnes problem is not complete, we add experiments to complement the theoretical results (i.e., Section~\ref{sec-exp}). We also add the robustness analysis of sampling to noise (i.e., Section~\ref{sec-sampling}). Note that the robustness of sampling to one-bit noise has been studied in our previous work~\cite{qian2016sampling}. It was shown that sampling can reduce the running time of the (1+1)-EA from exponential to polynomial on OneMax when the noise probability $p=1$ as well as on LeadingOnes when $p=1/2$. Therefore, our results here are more general. We prove that sampling is effective for any value of $p$, as shown in the last row of Table~\ref{table_runtime_sampling}. Furthermore, we analyze the robustness of sampling to bit-wise noise for the first time.

The rest of this paper is organized as follows. Section~2 introduces some preliminaries. The running time analysis of the (1+1)-EA on OneMax and LeadingOnes under noise is presented in Sections~3 and~4, respectively. Section~5 analyzes the (1+1)-EA using sampling. Section~6 concludes the paper.

\section{Preliminaries}
In this section, we first introduce the optimization problems, noise models and evolutionary algorithms studied in this paper, respectively, then introduce the sampling strategy, and finally present the analysis tools that we use throughout this paper.

\subsection{OneMax and LeadingOnes}

In this paper, we use two well-known pseudo-Boolean functions OneMax and LeadingOnes. The OneMax problem as presented in Definition~\ref{def_onemax} aims to maximize the number of 1-bits of a solution. The LeadingOnes problem as presented in Definition~\ref{def_leadingones} aims to maximize the number of consecutive 1-bits counting from the left of a solution. Their optimal solution is $11\ldots1$ (briefly denoted as $1^n$). It has been shown that the expected running time of the (1+1)-EA on OneMax and LeadingOnes is $\Theta(n \log n)$ and $\Theta(n^2)$, respectively~\cite{droste2002analysis}. In the following analysis, we will use $\mathrm{LO}(x)$ to denote the number of leading 1-bits of a solution $x$.

\begin{definition}[OneMax]\label{def_onemax}
    The OneMax Problem of size $n$ is to find an $n$ bits binary
    string $x^*$ such that\vspace{-0.3em}
    $$
        x^*=\mathop{\arg\max}\nolimits_{x \in \{0,1\}^n} \left( f(x)=\sum\nolimits^{n}_{i=1} x_i\right).
    $$
\end{definition}

\begin{definition}[LeadingOnes]\label{def_leadingones}
    The LeadingOnes Problem of size $n$ is to find an $n$ bits binary
    string $x^*$ such that\vspace{-0.3em}
    $$
        x^*=\mathop{\arg\max}\nolimits_{x \in \{0,1\}^n} \left(f(x)=\sum\nolimits^{n}_{i=1} \prod\nolimits^{i}_{j=1} x_j\right).
    $$
\end{definition}

\subsection{Bit-wise Noise}

There are mainly two kinds of noise models: prior and posterior~\cite{giessen2014robustness,jin2005evolutionary}. Let $f^n(x)$ and $f(x)$ denote the noisy and true fitness of a solution $x$, respectively. The prior noise comes from the variation on a solution, i.e., $f^n(x)=f(x')$, where $x'$ is generated from $x$ by random perturbations. The posterior noise comes from the variation on the fitness of a solution, e.g., additive noise $f^n(x)=f(x)+\delta$ and multiplicative noise $f^n(x)=f(x)\cdot \delta$, where $\delta$ is randomly drawn from some distribution. Previous theoretical analyses involving prior noise~\cite{dang2015efficient,droste2004analysis,giessen2014robustness,qian2016sampling,qian2015noise} often focused on a specific model, one-bit noise. As presented in Definition~\ref{one-bit-noise}, it flips a random bit of a solution before evaluation with probability $p$. However, in many realistic applications, noise can change several bits of a solution simultaneously rather than only one bit. We thus consider the bit-wise noise model. As presented in Definition~\ref{bit-wise-noise}, it happens with probability $p$, and independently flips each bit of a solution with probability $q$ before evaluation. We use bit-wise noise $(p,q)$ to denote the bit-wise noise model with a scenario of $(p,q)$.

\begin{definition}[One-bit Noise]\label{one-bit-noise}
Given a parameter $p \in [0,1]$, let $f^{n}(x)$ and $f(x)$ denote the noisy and true fitness of a solution $x \in \{0,1\}^n$, respectively, then
\begin{align*}
f^n(x)=\begin{cases}
f(x) & \text{with probability $1-p$},\\
f(x') & \text{with probability $p$},
\end{cases}
\end{align*}
where $x'$ is generated by flipping a uniformly randomly chosen bit of $x$.
\end{definition}

\begin{definition}[Bit-wise Noise]\label{bit-wise-noise}
Given parameters $p,q \!\in\! [0,1]$, let $f^{n}(x)$ and $f(x)$ denote the noisy and true fitness of a solution $x \!\in\! \{0,1\}^n$, respectively, then\vspace{-0.6em}
\begin{align*}
f^n(x)=\begin{cases}
f(x) & \text{with probability $1-p$},\\
f(x') & \text{with probability $p$},
\end{cases}
\end{align*}
where $x'$ is generated by independently flipping each bit of $x$ with probability $q$.
\end{definition}

To the best of our knowledge, only bit-wise noise $(1,q)$ has been recently studied. Gie{\ss}en and K{\"o}tzing~\cite{giessen2014robustness} proved that for the (1+1)-EA solving OneMax, the expected running time is polynomial if $q=O(\log n /n^2)$ and super-polynomial if $q=\omega(\log n /n^2)$. Besides bit-wise noise $(1,q)$, we also study another specific model bit-wise noise $(p,\frac{1}{n})$ in this paper. Note that bit-wise noise $(p,\frac{1}{n})$ is a natural extension of one-bit noise; their random behaviors of perturbing a solution correspond to the two common mutation operators, bit-wise mutation and one-bit mutation, respectively.

To investigate whether the performance of the (1+1)-EA for bit-wise noise with two scenarios of $(p,q)$ and $(p',q')$ where $p\cdot q=p'\cdot q'$ can be significantly different, we consider the OneMax problem under bit-wise noise $(1, \frac{\log n}{30n})$ and $(\frac{\log n}{30n}, 1)$. The comparison gives a positive answer. For bit-wise noise $(1, \frac{\log n}{30n})$, we know that the (1+1)-EA needs super-polynomial time to solve OneMax~\cite{giessen2014robustness}, while for bit-wise noise $(\frac{\log n}{30n}, 1)$, we will prove in Theorem~\ref{theo-equal-pq} that the (1+1)-EA can solve OneMax in polynomial time. Thus, the analysis on general bit-wise noise without fixing $p$ or $q$ would be complicated, and $p\cdot q$ may not be the only deciding factor. We leave it as a future work.

\subsection{(1+1) Evolutionary Algorithm}

The (1+1)-EA as described in Algorithm~\ref{(1+1)-EA} is studied in this paper. For noisy optimization, only a noisy fitness value $f^n(x)$ instead of the exact one $f(x)$ can be accessed, and thus line~4 of Algorithm~\ref{(1+1)-EA} is ``if {$f^{n}(x') \geq f^{n}(x)$}" instead of ``if {$f(x') \geq f(x)$}". Note that the reevaluation strategy is used as in~\cite{doerr2012ants,droste2004analysis,giessen2014robustness}. That is, besides evaluating $f^{n}(x')$, $f^{n}(x)$ will be reevaluated in each iteration of the (1+1)-EA. The running time is usually defined as the number of fitness evaluations needed to find an optimal solution w.r.t. the true fitness function $f$ for the first time~\cite{akimoto2015analysis,droste2004analysis,giessen2014robustness}.

\begin{algorithm}[(1+1)-EA]\label{(1+1)-EA} Given a function $f$ over $\{0,1\}^n$ to be maximized, it consists of the following steps:\\
    \begin{tabular}{ll}
    1. & $x:=$ uniformly randomly selected from $\{0,1\}^{n}$.\\
    2. & Repeat until the termination condition is met\\
    3. & \quad $x':=$ flip each bit of $x$ independently with prob. $1/n$. \\
    4. &\quad if {$f^n(x') \geq f^n(x)$} \\
    5. &\quad \quad $x:=x'$.
    \end{tabular}
\end{algorithm}

\subsection{Sampling}

In noisy evolutionary optimization, sampling has often been used to reduce the negative effect of noise~\cite{aizawa1994scheduling,branke2003selection}. As presented in Definition~\ref{sampling}, it approximates the true fitness $f(x)$ using the average of multiple ($m$) independent random evaluations. For the (1+1)-EA using sampling, line~4 of Algorithm~\ref{(1+1)-EA} changes to be ``if {$\hat{f}(x') \geq \hat{f}(x)$}". Its pseudo-code is described in Algorithm~\ref{(1+1)-EA-sample}. Note that the sample size $m=1$ is equivalent to that sampling is not used. The effectiveness of sampling was not theoretically analyzed until recently. Qian et al.~\cite{qian2016sampling} proved that sampling is robust to one-bit noise and additive Gaussian noise. Particularly, under one-bit noise, it was shown that sampling can reduce the running time exponentially for the (1+1)-EA solving OneMax when the noise probability $p=1$ and LeadingOnes when $p=1/2$.

\begin{definition}[Sampling]\label{sampling}
Sampling first evaluates the fitness of a solution $m$ times independently and obtains the noisy fitness values $f^n_1(x),\ldots,f^n_m(x)$, and then outputs their average, i.e.,
$$
 \hat{f}(x)=\frac{1}{m}\sum\nolimits^{m}_{i=1} f^{n}_i(x).
$$
\end{definition}

\begin{algorithm}[(1+1)-EA with sampling]\label{(1+1)-EA-sample} Given a function $f$ over $\{0,1\}^n$ to be maximized, it consists of the following steps:\\
    \begin{tabular}{ll}
    1. & $x:=$ uniformly randomly selected from $\{0,1\}^{n}$.\\
    2. & Repeat until the termination condition is met\\
    3. & \quad $x':=$ flip each bit of $x$ independently with prob. $1/n$. \\
    4. &\quad if {$\hat{f}(x') \geq \hat{f}(x)$} \\
    5. &\quad \quad $x:=x'$.
    \end{tabular}
\end{algorithm}

\subsection{Analysis Tools}

The process of the (1+1)-EA solving any pseudo-Boolean function with one unique global optimum can be directly modeled as a Markov chain $\{\xi_t\}^{+\infty}_{t=0}$. We only need to take the solution space $\{0,1\}^n$ as the chain's state space (i.e., $\xi_t \in \mathcal{X}=\{0,1\}^n$), and take the optimal solution $1^n$ as the chain's optimal state (i.e., $\mathcal{X}^*=\{1^n\}$). Note that we can assume without loss of generality that the optimal solution is $1^n$, because unbiased EAs treat the bits 0 and 1 symmetrically, and thus the 0 bits in an optimal solution can be interpreted as 1 bits without affecting the behavior of EAs. Given a Markov chain $\{\xi_t\}^{+\infty}_{t=0}$ and $\xi_{\hat{t}}=x$, we define its \emph{first hitting time} (FHT) as $\tau=\min\{t \mid \xi_{\hat{t}+t} \in \mathcal{X}^*,t\geq0\}$. The mathematical expectation of $\tau$, $\mathrm{E}(\tau \mid \xi_{\hat{t}}=x)=\sum\nolimits^{+\infty}_{i=0} i\cdot\mathrm{P}(\tau=i \mid \xi_{\hat{t}}=x)$, is called the \emph{expected first hitting time} (EFHT) starting from $\xi_{\hat{t}}=x$. If $\xi_{0}$ is drawn from a distribution $\pi_{0}$, $\mathrm{E}(\tau \mid \xi_{0}\sim \pi_0) = \sum\nolimits_{x\in \mathcal{X}} \pi_{0}(x)\mathrm{E}(\tau \mid \xi_{0}=x)$ is called the EFHT of the Markov chain over the initial distribution $\pi_0$. Thus, the expected running time of the (1+1)-EA starting from $\xi_0 \sim \pi_0$ is equal to $1+2\cdot \mathrm{E}(\tau \mid \xi_{0} \sim \pi_0)$, where the term 1 corresponds to evaluating the initial solution, and the factor 2 corresponds to evaluating the offspring solution $x'$ and reevaluating the parent solution $x$ in each iteration. If using sampling, the expected running time of the (1+1)-EA is $m+2m\cdot \mathrm{E}(\tau \mid \xi_{0} \sim \pi_0)$, since estimating the fitness of a solution needs $m$ number of independent fitness evaluations. Note that we consider the expected running time of the (1+1)-EA starting from a uniform initial distribution in this paper.

In the following, we give three drift theorems that will be used to analyze the EFHT of Markov chains in the paper. The additive drift theorem~\cite{he2001drift} as presented in Theorem~\ref{additive-drift} is used to derive upper bounds on the EFHT of Markov chains. To use it, a function $V(x)$ has to be constructed to measure the distance of a state $x$ to the optimal state space $\mathcal{X}^*$. Then, we need to investigate the progress on the distance to $\mathcal{X}^*$ in each step, i.e., $\mathbb{E}(V(\xi_t)-V(\xi_{t+1}) \mid \xi_t)$. An upper bound on the EFHT can be derived through dividing the initial distance by a lower bound on the progress.

\begin{theorem}[Additive Drift~\cite{he2001drift}]\label{additive-drift}
Given a Markov chain $\{\xi_t\}^{+\infty}_{t=0}$ and a distance function $V(x)$, if for any $t \geq 0$ and any $\xi_t$ with $V(\xi_t) > 0$, there exists a real number $c>0$ such that $$\mathrm{E}(V(\xi_t)-V(\xi_{t+1}) \mid \xi_t) \geq c,$$ then the EFHT satisfies that $\mathrm{E}(\tau \mid \xi_0) \leq V(\xi_0)/c.$
\end{theorem}

The negative drift theorem~\cite{oliveto2011simplified,oliveto2011simplifiedErratum} as presented in Theorem~\ref{simplified-drift} was proposed to prove exponential lower bounds on the EFHT of Markov chains, where $X_t$ is usually represented by a mapping of $\xi_t$. It requires two conditions: a constant negative drift and exponentially decaying probabilities of jumping towards or away from the goal state. To relax the requirement of a constant negative drift, the negative drift theorem with self-loops~\cite{jonathan2014offspring} as presented in Theorem~\ref{simplified-drift-selfloops} has been proposed, which takes into account large self-loop probabilities.

\begin{theorem}[Negative Drift~\cite{oliveto2011simplified,oliveto2011simplifiedErratum}]\label{simplified-drift}
Let $X_t$, $t\geq0$, be real-valued random variables describing a stochastic process. Suppose there exists an interval $[a,b]$ $\subseteq \mathbb{R}$, two constants $\delta,\epsilon>0$ and, possibly depending on $l:=b-a$, a function $r(l)$ satisfying $1\leq r(l)=o(l/\log(l))$ such that for all $t\geq 0$ the following two conditions hold:
\begin{align}
&1. \quad \mathrm{E}(X_t-X_{t+1} \mid a < X_t <b) \leq -\epsilon,\\
&2. \quad \mathrm{P}(|X_{t+1}-X_t| \geq j \mid X_t>a) \leq \frac{r(l)}{(1+\delta)^j} \;\; \text{for} \; j \in \mathbb{N}_0.
\end{align}
Then there is a constant $c>0$ such that for $T:=\min\{t \geq 0: X_t \leq a \mid X_0 \geq b\}$ it holds $\mathrm{P}(T \leq 2^{cl/r(l)})=2^{-\Omega(l/r(l))}$.
\end{theorem}

\begin{theorem}[Negative Drift with Self-loops~\cite{jonathan2014offspring}]\label{simplified-drift-selfloops}
Let $X_t$, $t\geq0$, be real-valued random variables describing a stochastic process. Suppose there exists an interval $[a,b] \subseteq \mathbb{R}$, two constants $\delta,\epsilon>0$ and, possibly depending on $l:=b-a$, a function $r(l)$ satisfying $1\leq r(l)=o(l/\log(l))$ such that for all $t\geq 0$ the following two conditions hold:
\begin{align}
1.&\; \forall a<i<b: \mathrm{E}(X_t-X_{t+1} \mid X_t=i) \leq -\epsilon \cdot \mathrm{P}(X_{t+1} \neq i \mid X_t=i),\\
2.&\; \forall i \!>\!a, j \!\in\! \mathbb{N}_0: \mathrm{P}(|X_{t+1}\!-\!X_t| \!\geq\! j \mid X_t\!=\!i) \leq \frac{r(l)}{(1\!+\!\delta)^j}\cdot \mathrm{P}(X_{t+1} \!\neq\! i \mid X_t\!=\!i).
\end{align}
Then there is a constant $c>0$ such that for $T:=\min\{t \geq 0: X_t \leq a \mid X_0 \geq b\}$ it holds $\mathrm{P}(T \leq 2^{cl/r(l)})=2^{-\Omega(l/r(l))}$.
\end{theorem}

\section{The OneMax problem}

In this section, we analyze the running time of the (1+1)-EA on OneMax under bit-wise noise. Note that for bit-wise noise $(1,q)$, it has been proved that the expected running time is polynomial if and only if $q=O(\log n/n^2)$, as shown in Theorem~\ref{onemax-noise-1}.

\begin{theorem}\cite{giessen2014robustness}\label{onemax-noise-1}
For the (1+1)-EA on OneMax under bit-wise noise $(1,q)$, the expected running time is polynomial if $q =O(\log n/n^2)$ and super-polynomial if $q=\omega(\log n / n^2)$.
\end{theorem}

For bit-wise noise $(p,\frac{1}{n})$, we prove in Theorems~\ref{onemax-noise-2-poly} and~\ref{onemax-noise-2-superpoly} that the tight range of $p$ allowing a polynomial running time is $O(\log n/n)$. Instead of using the original drift theorems, we apply the upper and lower bounds of the (1+1)-EA on noisy OneMax in~\cite{giessen2014robustness}. Let $x^k$ denote any solution with $k$ number of 1-bits, and $f^n(x^{k})$ denote its noisy objective value, which is a random variable. Lemma~\ref{onemax-noise-upper} intuitively means that if the probability of recognizing the true better solution by noisy evaluation is large (i.e., Eq.~(\refeq{eq-upper-cond})), the running time can be upper bounded. Particularly, if Eq.~(\refeq{eq-upper-cond}) holds with $l=O(\log n)$, the running time can be polynomially upper bounded. On the contrary, Lemma~\ref{onemax-noise-lower} shows that if the probability of making a right comparison is small (i.e., Eq.~(\refeq{eq-lower-cond})), the running time can be lower bounded. Particularly, if Eq.~(\refeq{eq-lower-cond}) holds with $l=\Omega(n)$, the running time can be exponentially lower bounded. Both Lemmas~\ref{onemax-noise-upper} and~\ref{onemax-noise-lower} are proved by applying standard drift theorems, and can be used to simplify our analysis. Note that in the original upper bound of the (1+1)-EA on noisy OneMax (i.e., Theorem~5 in~\cite{giessen2014robustness}), it requires that Eq.~(\refeq{eq-upper-cond}) holds with only $j=k$, but the proof actually also requires that noisy OneMax satisfies the monotonicity property, i.e., for all $j < k<n$, $\mathrm{P}(f^n(x^k)<f^n(x^{k+1})) \leq \mathrm{P}(f^n(x^j)<f^n(x^{k+1}))$. We have combined these two conditions in Lemma~\ref{onemax-noise-upper} by requiring Eq.~(\refeq{eq-upper-cond}) to hold with any $j \leq k$ instead of only $j=k$.

\begin{lemma}\cite{giessen2014robustness}\label{onemax-noise-upper}
Suppose there is a positive constant $c \leq 1/15$ and some $2 < l \leq n/2$ such that
\begin{equation}
\begin{aligned}\label{eq-upper-cond}
& \forall j\leq k<n: \mathrm{P}(f^n(x^j) < f^n(x^{k+1})) \geq 1-\frac{l}{n};\\
& \forall j \leq k<n-l: \mathrm{P}(f^n(x^j) < f^n(x^{k+1})) \geq 1-c\frac{n-k}{n},
\end{aligned}
\end{equation}
then the (1+1)-EA optimizes $f$ in expectation in $O(n \log n) +n2^{O(l)}$ iterations.
\end{lemma}

\begin{lemma}\cite{giessen2014robustness}\label{onemax-noise-lower}
Suppose there is some $l \leq n/4$ and a constant $c \geq 16$ such that
\begin{equation}
\begin{aligned}\label{eq-lower-cond}
&\forall n-l \leq k < n: \mathrm{P}(f^n(x^k) < f^n(x^{k+1})) \leq 1-c\frac{n-k}{n},
\end{aligned}
\end{equation}
then the (1+1)-EA optimizes $f$ in $2^{\Omega(l)}$ iterations with a high probability.
\end{lemma}

First, we apply Lemma~\ref{onemax-noise-upper} to show that the expected running time is polynomially upper bounded for bit-wise noise $(p,\frac{1}{n})$ with $p=O(\log n/n)$.

\begin{theorem}\label{onemax-noise-2-poly}
For the (1+1)-EA on OneMax under bit-wise noise $(p,\frac{1}{n})$, the expected running time is polynomial if $p =O(\log n/n)$.
\end{theorem}
\begin{myproof}
We prove it by using Lemma~\ref{onemax-noise-upper}. For some positive constant $b$, suppose that $p \leq b\log n/n$. We set the two parameters in Lemma~\ref{onemax-noise-upper} as $c=\min\{\frac{1}{15},b\}$ and $l=\frac{2b\log n}{c} \in (2,\frac{n}{2}]$.

For any $j \leq k<n$, $f^n(x^j) \geq f^n(x^{k+1})$ implies that $f^n(x^j) \geq k+1$ or $f^n(x^{k+1}) \leq k$, either of which happens with probability at most $p$. By the union bound, we get $\forall j \leq k<n$,
$$
 \mathrm{P}(f^n(x^j) \geq f^n(x^{k+1})) \leq 2p \leq \frac{2b\log n}{n}=\frac{lc}{n}\leq \frac{l}{n}.
$$
For any $j \leq k < n-l$, we easily get
$$
\mathrm{P}(f^n(x^j) \geq f^n(x^{k+1})) \leq \frac{lc}{n}< c\frac{n-k}{n}.
$$

By Lemma~\ref{onemax-noise-upper}, we know that the expected running time is $O(n \log n)+n2^{O(2b\log n/c)}$, i.e., polynomial.\vspace{0.8em}
\end{myproof}

Next we apply Lemma~\ref{onemax-noise-lower} to show that the expected running time is super-polynomial for bit-wise noise $(p,\frac{1}{n})$ with $p=\omega(\log n/n)$. Note that for $p=1-O(\log n/n)$, we actually give a stronger result that the expected running time is exponential.

\begin{theorem}\label{onemax-noise-2-superpoly}
For the (1+1)-EA on OneMax under bit-wise noise $(p,\frac{1}{n})$, the expected running time is super-polynomial if $p =\omega(\log n/n) \cap 1-\omega(\log n/n)$ and exponential if $p =1-O(\log n/n)$.
\end{theorem}
\begin{myproof}
We use Lemma~\ref{onemax-noise-lower} to prove it. Let $c = 16$. The case $p =\omega(\log n/n) \cap 1-\omega(\log n/n)$ is first analyzed. For any positive constant $b$, let $l=b\log n$. For any $k \geq n-l$, we get
$$
\mathrm{P}(f^n(x^k) \geq f^{n}(x^{k+1})) \geq \mathrm{P}(f^n(x^k)=k) \cdot \mathrm{P}(f^n(x^{k+1})\leq k).
$$
To make $f^n(x^k)=k$, it is sufficient that the noise does not happen, i.e., $\mathrm{P}(f^n(x^k)=k) \geq 1-p$. To make $f^n(x^{k+1})\leq k$, it is sufficient to flip one 1-bit and keep other bits unchanged by noise, i.e., $\mathrm{P}(f^n(x^{k+1})\leq k) \geq p\cdot \frac{k+1}{n}(1-\frac{1}{n})^{n-1}$. Thus,
$$
\mathrm{P}(f^n(x^k) \geq f^{n}(x^{k+1})) \geq (1-p)\cdot p\frac{k+1}{en}=\omega(\log n/n).
$$
Since $c\frac{n-k}{n} \leq c\frac{l}{n}=\frac{cb\log n}{n}$, the condition of Lemma~\ref{onemax-noise-lower} holds. Thus, the expected running time is $2^{\Omega(b\log n)}$ (where $b$ is any constant), i.e., super-polynomial.

For the case $p =1-O(\log n/n)$, let $l=\sqrt{n}$. We use another lower bound $p(1-\frac{1}{n})^{n}$ for $\mathrm{P}(f^n(x^k)=k)$, since it is sufficient that no bit flips by noise. Thus, we have
$$
\mathrm{P}(f^n(x^k) \geq f^{n}(x^{k+1})) \geq p\left(1-\frac{1}{n}\right)^{n}\cdot p \frac{k+1}{en}=\Omega(1).
$$
Since $c\frac{n-k}{n} \leq \frac{c\sqrt{n}}{n}$, the condition of Lemma~\ref{onemax-noise-lower} holds. Thus, the expected running time is $2^{\Omega(\sqrt{n})}$, i.e., exponential.\vspace{0.8em}
\end{myproof}

To show that the performance of the (1+1)-EA for bit-wise noise with two scenarios $(p,q)$ and $(p',q')$ where $p\cdot q=p'\cdot q'$ can be significantly different, we compare the expected running time of the (1+1)-EA for bit-wise noise $ (1,\frac{\log n}{30n}) $ and $(\frac{\log n}{30n},1)$. For the former case, we know from Theorem~\ref{onemax-noise-1} that the expected running time is super-polynomial, while for the latter case, we prove in the following theorem that the expected running time can be polynomially upper bounded.

\begin{theorem}\label{theo-equal-pq}
For the (1+1)-EA on OneMax under bit-wise noise $(\frac{\log n}{30n},1)$, the expected running time is polynomial.
\end{theorem}
\begin{myproof}
We use Lemma~\ref{onemax-noise-upper} to prove it. For any $ j\le k<n $, $ f^n(x^j)\ge f^n(x^{k+1}) $ implies that the fitness evaluation of $ x^j $ or $ x^{k+1} $ is affected by noise, whose probability is at most $ 2\cdot \frac{\log n}{30n} =\frac{\log n}{15n}$. Thus, we have $ \mathrm{P}(f^n(x^j)<f^n(x^{k+1}))\ge 1-\frac{\log n}{15n}$. It is then easy to verify that the condition of Lemma~\ref{onemax-noise-upper} holds with $ c=\frac{1}{15} $ and $ l=\log n $. Thus, the expected running time is polynomial.
\end{myproof}

\section{The LeadingOnes problem}

In this section, we first analyze the running time of the (1+1)-EA on the LeadingOnes problem under bit-wise noise $(p,\frac{1}{n})$ and bit-wise noise $(1,q)$, respectively. Then, we transfer the analysis from bit-wise noise $(p,\frac{1}{n})$ to one-bit noise; the results are complementary to the known ones recently derived in~\cite{giessen2014robustness}. However, our analysis does not cover all the ranges of $p$ and $q$. For those values of $p$ and $q$ where no theoretical results are known, we conduct experiments to empirically investigate the running time.

\subsection{Bit-wise Noise $(p,\frac{1}{n})$}

For bit-wise noise $(p,\frac{1}{n})$, we first apply the additive drift theorem (i.e., Theorem~\ref{additive-drift}) to prove that the expected running time is polynomial when $p=O(\log n/n^2)$.

\begin{theorem}\label{leadingones-noise-2-poly}
For the (1+1)-EA on LeadingOnes under bit-wise noise $(p,\frac{1}{n})$, the expected running time is polynomial if $p =O(\log n/n^2)$.
\end{theorem}
\begin{myproof}
We use Theorem~\ref{additive-drift} to prove it. For some positive constant $b$, suppose that $p \leq b\log n /n^2$. Let $\theta \in (0,1)$ be some constant close to 0. We first construct a distance function $V(x)$ as, for any $x$ with $\mathrm{LO}(x)=i$,
\begin{align}\label{eq-distance-1}
V(x)=\left(1+\frac{c}{n}\right)^n-\left(1+\frac{c}{n}\right)^i,
\end{align}
where $c=\frac{2b\log n}{1-\theta}+1$. It is easy to verify that $V(x \in \mathcal{X}^*=\{1^n\})=0$ and $V(x \notin \mathcal{X}^*)>0$.

Then, we investigate $\mathrm{E}(V(\xi_t)-V(\xi_{t+1}) \mid \xi_t=x)$ for any $x$ with $\mathrm{LO}(x)<n$ (i.e., $x \notin \mathcal{X^*}$). Assume that currently $\mathrm{LO}(x)=i$, where $0\leq i \leq n-1$. Let $\mathrm{P}_{mut}(x,x')$ denote the probability of generating $x'$ by mutation on $x$. We divide the drift into two parts: positive $\mathrm{E}^+$ and negative $\mathrm{E}^-$. That is, $$\mathrm{E}(V(\xi_t)-V(\xi_{t+1}) \mid \xi_t=x)=\mathrm{E}^+-\mathrm{E}^-,$$ where
\begin{equation}
\begin{aligned}\label{eq-positive-drift}
&\mathrm{E}^+=\sum_{x': \mathrm{LO}(x')>i}\mathrm{P}_{mut}(x,x')\cdot \mathrm{P}(f^n(x') \geq f^n(x)) \cdot (V(x)-V(x')),
\end{aligned}
\end{equation}
\begin{equation}
\begin{aligned}\label{eq-negative-drift}
&\mathrm{E}^-=\sum_{x': \mathrm{LO}(x')<i}\mathrm{P}_{mut}(x,x')\cdot \mathrm{P}(f^n(x') \geq f^n(x))\cdot (V(x')-V(x)).
\end{aligned}
\end{equation}

For the positive drift $\mathrm{E}^+$, we need to consider that the number of leading 1-bits is increased. By mutation, we have
\begin{equation}
\begin{aligned}\label{eq-mut-1}
&\sum_{x': \mathrm{LO}(x')>i}\mathrm{P}_{mut}(x,x')=\mathrm{P}(\mathrm{LO}(x') \geq i+1)=\left(1-\frac{1}{n}\right)^i\frac{1}{n},
\end{aligned}
\end{equation}
since it needs to flip the $(i+1)$-th bit (which must be 0) of $x$ and keep the $i$ leading 1-bits unchanged. For any $x'$ with $\mathrm{LO}(x') \geq i+1$, $f^n(x') < f^n(x)$ implies that $f^n(x') \leq i-1$ or $f^n(x) \geq i+1$. Note that,
\begin{equation}
\begin{aligned}\label{eq-bit-wise-1}
&\mathrm{P}(f^n(x') \leq i-1)=p\left(1-\left(1-\frac{1}{n}\right)^i\right),
\end{aligned}
\end{equation}
since at least one of the first $i$ leading 1-bits of $x'$ needs to be flipped by noise;
\begin{equation}
\begin{aligned}\label{eq-bit-wise-2}
&\mathrm{P}(f^n(x) \geq i+1)=p\left(1-\frac{1}{n}\right)^i\frac{1}{n},
\end{aligned}
\end{equation}
since it needs to flip the first 0-bit of $x$ and keep the leading 1-bits unchanged by noise. By the union bound, we get
\begin{equation}
\begin{aligned}\label{eq-bit-wise-s1}
&\mathrm{P}(f^n(x') \geq f^n(x)) = 1- \mathrm{P}(f^n(x') < f^n(x))\\
&\geq 1-p\left(1-\left(1-\frac{1}{n}\right)^{i+1}\right) \geq 1-p\frac{i+1}{n} \geq 1-p \geq 1-\theta,
\end{aligned}
\end{equation}
where the last inequality holds with sufficiently large $n$, since $p=O(\log n/n^2)$ and $\theta \in (0,1)$ is some constant close to 0. Furthermore, for any $x'$ with $V(x')\geq i+1$,
\begin{equation}
\begin{aligned}\label{eq-drift-1}
&V(x)-V(x') \geq \left(1+\frac{c}{n}\right)^{i+1}-\left(1+\frac{c}{n}\right)^i=\frac{c}{n}\left(1+\frac{c}{n}\right)^i.
\end{aligned}
\end{equation}
By combining Eqs.~(\refeq{eq-mut-1}),~(\refeq{eq-bit-wise-s1}) and~(\refeq{eq-drift-1}), we have
\begin{equation}
\begin{aligned}
&\mathrm{E}^+ \geq \left(1-\frac{1}{n}\right)^i\frac{1}{n} \cdot (1-\theta) \cdot \frac{c}{n}\left(1+\frac{c}{n}\right)^i \geq \frac{(1-\theta)c}{3n^2}\left(1+\frac{c}{n}\right)^i,
\end{aligned}
\end{equation}
where the last inequality is by $(1-\frac{1}{n})^{i} \geq (1-\frac{1}{n})^{n-1} \geq \frac{1}{e}\geq \frac{1}{3}$.

For the negative drift $\mathrm{E}^-$, we need to consider that the number of leading 1-bits is decreased. By mutation, we have
\begin{equation}
\begin{aligned}\label{eq-mut-2}
&\sum_{x': \mathrm{LO}(x')<i}\mathrm{P}_{mut}(x,x')=\mathrm{P}(\mathrm{LO}(x') \leq i-1)=1-\left(1-\frac{1}{n}\right)^i,
\end{aligned}
\end{equation}
since it needs to flip at least one leading 1-bit of $x$. For any $x'$ with $\mathrm{LO}(x') \leq i-1$ (where $i \geq 1$), $f^n(x') \geq f^n(x)$ implies that $f^n(x') \geq i$ or $f^n(x) \leq i-1$. Note that,
\begin{align}\label{eq-bit-wise-3}
&\mathrm{P}(f^n(x') \geq i)\leq p\left(1-\frac{1}{n}\right)^{i-1}\frac{1}{n},
\end{align}
since for the first $i$ bits of $x'$, it needs to flip the 0-bits (whose number is at least 1) and keep the 1-bits unchanged by noise;
\begin{align}\label{eq-bit-wise-4}
&\mathrm{P}(f^n(x) \leq i-1)=p\left(1-\left(1-\frac{1}{n}\right)^i\right),
\end{align}
since at least one leading 1-bit of $x$ needs to be flipped by noise. By the union bound, we get
\begin{align}\label{eq-bit-wise-s2}
&\mathrm{P}(f^n(x') \geq f^n(x)) \leq p-p\left(1-\frac{2}{n}\right)\left(1-\frac{1}{n}\right)^{i-1} \leq p\frac{i+1}{n}.
\end{align}
Furthermore, according to the definition of the distance function (i.e., Eq.~(\refeq{eq-distance-1})), we have for any $x'$ with $0 \leq \mathrm{LO}(x') \leq i-1$,
\begin{align}\label{eq-drift-2}
&V(x')-V(x) =\left(1+\frac{c}{n}\right)^i-\left(1+\frac{c}{n}\right)^{\mathrm{LO}(x')}\leq \left(1+\frac{c}{n}\right)^i-1.
\end{align}
By combining Eqs.~(\refeq{eq-mut-2}),~(\refeq{eq-bit-wise-s2}) and~(\refeq{eq-drift-2}), we have
\begin{align}
\mathrm{E}^- &\leq \left(1-\left(1-\frac{1}{n}\right)^i\right)\cdot p\frac{i+1}{n} \cdot \left(\left(1+\frac{c}{n}\right)^i-1\right)  \\
&\leq \left(1-\frac{1}{e}\right) \cdot p \cdot \left(1+\frac{c}{n}\right)^i\leq \frac{2p}{3}\left(1+\frac{c}{n}\right)^i.
\end{align}

Thus, by subtracting $\mathrm{E}^-$ from $\mathrm{E}^+$, we have
\begin{align}\label{eq-sum-1}
&\mathrm{E}(V(\xi_t)-V(\xi_{t+1}) \mid \xi_t=x)\geq  \left(1+\frac{c}{n}\right)^i  \left(\frac{(1-\theta)c}{3n^2}-\frac{2p}{3}\right)\\
& \geq \left(1+\frac{c}{n}\right)^i \left(\frac{2b\log n+1-\theta}{3n^2}-\frac{2b \log n}{3n^2}\right) \geq \frac{1-\theta}{3n^2},
\end{align}
where the second inequality is by $c=\frac{2b\log n}{1-\theta}+1$ and $p \leq b \log n/n^2$. Note that $V(x)\leq (1+\frac{c}{n})^n \leq e^c=e^{\frac{2b\log n}{1-\theta}+1}=en^{\frac{2b}{1-\theta}}$. By Theorem~\ref{additive-drift}, we get
$$
\mathrm{E}(\tau \mid \xi_0) \leq \frac{3n^2}{1-\theta} \cdot en^{\frac{2b}{1-\theta}}=O\left(n^{\frac{2b}{1-\theta}+2}\right),
$$
i.e., the expected running time is polynomial.\vspace{0.8em}
\end{myproof}

Next we use the negative drift with self-loops theorem (i.e., Theorem~\ref{simplified-drift-selfloops}) to show that the expected running time is super-polynomial for bit-wise noise $(p,\frac{1}{n})$ with $p=\omega(\log n/n) \cap o(1)$.

\begin{theorem}\label{leadingones-noise-2-superpoly}
For the (1+1)-EA on LeadingOnes under bit-wise noise $(p,\frac{1}{n})$, if $p =\omega(\log n/n) \cap o(1)$, the expected running time is super-polynomial.
\end{theorem}
\begin{myproof}
We use Theorem~\ref{simplified-drift-selfloops} to prove it. Let $X_t=|x|_0$ be the number of 0-bits of the solution $x$ after $t$ iterations of the (1+1)-EA. Let $c$ be any positive constant. We consider the interval $[0,c \log n]$, i.e., the parameters $a=0$ (i.e., the global optimum) and $b=c \log n$ in Theorem~\ref{simplified-drift-selfloops}.

Then, we analyze the drift $\mathrm{E}(X_t-X_{t+1} \mid X_t=i)$ for $1\leq i<c \log n$. As in the proof of Theorem~\ref{leadingones-noise-2-poly}, we divide the drift into two parts: positive $\mathrm{E}^+$ and negative $\mathrm{E}^-$. That is, $$\mathrm{E}(X_t-X_{t+1} \mid X_t=i)=\mathrm{E}^+-\mathrm{E}^-,$$ where
\begin{equation}
\begin{aligned}
&\mathrm{E}^+=\sum_{x': |x'|_0<i}\mathrm{P}_{mut}(x,x')\cdot \mathrm{P}(f^n(x') \geq f^n(x)) \cdot (i-|x'|_0),\\
&\mathrm{E}^-=\sum_{x': |x'|_0>i}\mathrm{P}_{mut}(x,x')\cdot \mathrm{P}(f^n(x') \geq f^n(x))\cdot (|x'|_0-i).
\end{aligned}
\end{equation}
Note that the drift here depends on the number of 0-bits due to the definition of $X_t$. It is different from that in the proof of Theorem~\ref{leadingones-noise-2-poly}, which depends on the number of leading 1-bits due to the definition of the distance function (i.e., Eq.~(\refeq{eq-distance-1})).

For the positive drift $\mathrm{E}^+$, we need to consider that the number of 0-bits is decreased. For mutation on $x$ (where $|x|_0=i$), let $X$ and $Y$ denote the number of flipped 0-bits and 1-bits, respectively. Then, $X \sim B(i,\frac{1}{n})$ and $Y \sim B(n-i,\frac{1}{n})$, where $B(\cdot,\cdot)$ is the binomial distribution. To estimate an upper bound on $\mathrm{E}^+$, we assume that the offspring solution $x'$ with $|x'|_0 <i$ is always accepted. Thus, we have
\begin{align}\label{eq-positive-drift-1}
\mathrm{E}^+& \leq \sum_{x':|x'|_0 <i} \mathrm{P}_{mut}(x,x')(i-|x'|_0)=\sum^i_{k=1} k \cdot \mathrm{P}(X-Y=k)\\
&=\sum\nolimits^i_{k=1} k \cdot \sum\nolimits^i_{j=k} \mathrm{P}(X=j) \cdot \mathrm{P}(Y=j-k)\\
&=\sum\nolimits^i_{j=1} \sum\nolimits^j_{k=1} k \cdot \mathrm{P}(X=j) \cdot \mathrm{P}(Y=j-k)\\
&\leq \sum\nolimits^i_{j=1}  j \cdot \mathrm{P}(X=j)=\frac{i}{n}.
\end{align}

For the negative drift $\mathrm{E}^-$, we need to consider that the number of 0-bits is increased. We analyze the $n-i$ cases where only one 1-bit is flipped (i.e., $|x'|_0=i+1$), which happens with probability $\frac{1}{n}(1-\frac{1}{n})^{n-1}$. Assume that $\mathrm{LO}(x)=k \leq n-i$. If the $j$-th (where $1 \leq j \leq k$) leading 1-bit is flipped, the offspring solution $x'$ will be accepted (i.e., $f^n(x') \geq f^n(x)$) if $f^n(x') \geq j-1$ and $f^n(x) \leq j-1$. Note that,
\begin{equation}
\begin{aligned}\label{eq-bit-wise-5}
\mathrm{P}(f^n(x') \geq j-1) = 1-p+p\left(1-\frac{1}{n}\right)^{j-1}\geq 1-p\frac{j-1}{n} \geq \frac{1}{2},
\end{aligned}
\end{equation}
where the equality is since it needs to keep the $j-1$ leading 1-bits of $x'$ unchanged, and the last inequality is by $p=o(1)$;
\begin{align}\label{eq-bit-wise-6}
&\mathrm{P}(f^n(x) \leq j-1)=p\left(1-\left(1-\frac{1}{n}\right)^{j}\right)\\
&=p\left(1-\frac{1}{n}\right)^j\left(\left(1+\frac{1}{n-1}\right)^j-1\right) \geq \frac{p}{e} \cdot \frac{j}{n-1} \geq \frac{pj}{3n},
\end{align}
where the equality is since at least one of the first $j$ leading 1-bits of $x$ needs to be flipped by noise. Thus, we get
\begin{align}\label{eq-bit-wise-s3}
&\mathrm{P}(f^n(x') \geq f^n(x)) \geq \frac{pj}{6n}.
\end{align}
If one of the $n-i-k$ non-leading 1-bits is flipped, $\mathrm{LO}(x')=\mathrm{LO}(x)=k$. We can use the same analysis procedure as Eq.~(\refeq{eq-bit-wise-s1}) in the proof of Theorem~\ref{leadingones-noise-2-poly} to derive that
\begin{align}\label{eq-bit-wise-s4}
&\mathrm{P}(f^n(x') \geq f^n(x)) \geq 1-p\frac{k+1}{n} \geq \frac{1}{2},
\end{align}
where the last inequality is by $p=o(1)$. Combining all the $n-i$ cases, we get
\begin{align}\label{eq-drift-3}
\mathrm{E}^- &\geq \frac{1}{n}\left(1-\frac{1}{n}\right)^{n-1} \cdot \left(\sum^k_{j=1} \frac{pj}{6n}+\frac{n-i-k}{2}\right) \cdot (i+1-i)\\
&\geq \frac{1}{en}\left(\frac{pk(k+1)}{12n}+\frac{n-i-k}{2}\right) \geq \frac{pk^2}{36n^2}+\frac{n-i-k}{6n}.
\end{align}

By subtracting $\mathrm{E}^-$ from $\mathrm{E}^+$, we get
$$\mathrm{E}(X_t-X_{t+1} \mid X_t=i)\leq \frac{i}{n}-\frac{pk^2}{36n^2}-\frac{n-i-k}{6n}.$$ To investigate condition (1) of Theorem~\ref{simplified-drift-selfloops}, we also need to analyze the probability $\mathrm{P}(X_{t+1} \neq i \mid X_t=i)$. For $X_{t+1}\neq i$, it is necessary that at least one bit of $x$ is flipped and the offspring $x'$ is accepted. We consider two cases: (1) at least one of the $k$ leading 1-bits of $x$ is flipped; (2) the $k$ leading 1-bits of $x$ are not flipped and at least one of the last $n-k$ bits is flipped. For case~(1), the mutation probability is $1-(1-\frac{1}{n})^k$ and the acceptance probability is at most $p\frac{k+1}{n}$ by Eq.~(\refeq{eq-bit-wise-s2}). For case~(2), the mutation probability is $(1-\frac{1}{n})^{k}(1-(1-\frac{1}{n})^{n-k}) \leq \frac{n-k}{n}$ and the acceptance probability is at most $1$. Thus, we have
\begin{align}\label{eq-mid1}
&\mathrm{P}(X_{t+1} \neq i \mid X_t=i) \leq p+\frac{n-k}{n}.
\end{align}
When $k<n-np$, we have
\begin{align}\label{eq-mid2}
&\mathrm{E}(X_t-X_{t+1} \mid X_t=i)\leq  \frac{i}{n}-\frac{n-i-k}{6n}\\
&\leq -\frac{n-k}{12n}-\frac{np/2-7c\log n}{6n}\leq -\frac{n-k}{12n} \leq  -\frac{1}{24}\left(p+\frac{n-k}{n}\right),
\end{align}
where the second inequality is by $n-k>np$ and $i< c \log n$, the third inequality is by $p =\omega(\log n/n)$ and the last is by $n-k>np$. When $k\geq n-np$, we have
\begin{align}\label{eq-mid3}
&\mathrm{E}(X_t-X_{t+1} \mid X_t=i)\leq  \frac{i}{n}-\frac{pk^2}{36n^2}\\
& \leq \frac{c\log n}{n}-\frac{p}{144}\leq -\frac{p}{288} \leq -\frac{1}{576}\left(p+\frac{n-k}{n}\right),
\end{align}
where the second inequality is by $p=o(1)$ and $i< c \log n$, the third is by $p=\omega(\log n/n)$ and the last is by $n-k \leq np$. Combining Eqs.~(\refeq{eq-mid1}),~(\refeq{eq-mid2}) and~(\refeq{eq-mid3}), we get that condition (1) of Theorem~\ref{simplified-drift-selfloops} holds with $\epsilon =\frac{1}{576}$.

For condition (2) of Theorem~\ref{simplified-drift-selfloops}, we need to show $\mathrm{P}(|X_{t+1}-X_{t}|\geq j \mid X_t =i) \leq \frac{r(l)}{(1+\delta)^j}\cdot \mathrm{P}(X_{t+1} \neq i \mid X_t=i)$ for $i \geq 1$. For $\mathrm{P}(X_{t+1} \neq i \mid X_t=i)$, we analyze the $n$ cases where only one bit is flipped. Using the similar analysis procedure as $\mathrm{E}^-$, except that flipping any bit rather than only 1-bit is considered here, we easily get
\begin{align}\label{eq-mid4}
&\mathrm{P}(X_{t+1} \neq i \mid X_t=i) \geq \frac{pk(k+1)}{36n^2}+\frac{n-k}{6n}.
\end{align}
For $|X_{t+1}-X_{t}|\geq j$, it is necessary that at least $j$ bits of $x$ are flipped and the offspring solution $x'$ is accepted. We consider two cases: (1) at least one of the $k$ leading 1-bits is flipped; (2) the $k$ leading 1-bits are not flipped. For case~(1), the mutation probability is at most $\frac{k}{n}\binom{n-1}{j-1}\frac{1}{n^{j-1}}$ and the acceptance probability is at most $p\frac{k+1}{n}$ by Eq.~(\refeq{eq-bit-wise-s2}). For case~(2), the mutation probability is at most $(1-\frac{1}{n})^k\binom{n-k}{j}\frac{1}{n^{j}}$ and the acceptance probability is at most 1. Thus, we have
\begin{align}\label{eq-mid5}
&\mathrm{P}(|X_{t+1}-X_{t}|\geq j \mid X_t =i) \\
&\leq \frac{k}{n}\binom{n-1}{j-1}\frac{1}{n^{j-1}} \cdot p\frac{k+1}{n} +\left(1-\frac{1}{n}\right)^k\binom{n-k}{j}\frac{1}{n^{j}}\\
& \leq \frac{pk(k+1)}{n^2} \cdot \frac{4}{2^j}+\frac{n-k}{n}\cdot \frac{2}{2^j} \leq \left(\frac{pk(k+1)}{36n^2}+\frac{n-k}{6n}\right)\cdot \frac{144}{2^j}.
\end{align}
By combining Eq.~(\refeq{eq-mid4}) with Eq.~(\refeq{eq-mid5}), we get that condition (2) of Theorem~\ref{simplified-drift-selfloops} holds with $\delta=1$ and $r(l)=144$.

Note that $l=b-a=c \log n$. By Theorem~\ref{simplified-drift-selfloops}, the expected running time is $2^{\Omega(c \log n)}$, where $c$ is any positive constant. Thus, the expected running time is super-polynomial.\vspace{0.8em}
\end{myproof}

For $p=\Omega(1)$, we can use the negative drift theorem (i.e., Theorem~\ref{simplified-drift}) to derive a stronger result that the expected running time is exponentially lower bounded.

\begin{theorem}\label{leadingones-noise-2-exp}
For the (1+1)-EA on LeadingOnes under bit-wise noise $(p,\frac{1}{n})$, the expected running time is exponential if $p=\Omega(1)$.
\end{theorem}
\begin{myproof}
We use Theorem~\ref{simplified-drift} to prove it. Let $X_t=i$ be the number of 0-bits of the solution $x$ after $t$ iterations of the (1+1)-EA. We consider the interval $i \in [0,n^{1/2}]$. To analyze the drift $\mathrm{E}(X_t-X_{t+1} \mid X_t=i)=\mathrm{E}^+-\mathrm{E}^-$, we use the same analysis procedure as Theorem~\ref{leadingones-noise-2-superpoly}. For the positive drift, we have $\mathrm{E}^+ \leq \frac{i}{n}=o(1)$. For the negative drift, we re-analyze Eqs.~(\ref{eq-bit-wise-s3}) and~(\ref{eq-bit-wise-s4}). From Eqs.~(\ref{eq-bit-wise-5}) and~(\ref{eq-bit-wise-6}), we get that $\mathrm{P}(f^n(x') \geq j-1) \geq p(1-\frac{j-1}{n})$ and $\mathrm{P}(f^n(x) \leq j-1) \geq \frac{pj}{3n}$. Thus, Eq.~(\ref{eq-bit-wise-s3}) becomes
\begin{align}\label{eq-bit-wise-s5}
&\mathrm{P}(f^n(x') \geq f^n(x)) \geq \frac{p^2j}{3n}\left(1-\frac{j-1}{n}\right).
\end{align}
For Eq.~(\ref{eq-bit-wise-s4}), we need to analyze the acceptance probability for $\mathrm{LO}(x')=\mathrm{LO}(x)=k$. Since it is sufficient to keep the first $(k+1)$ bits of $x$ and $x'$ unchanged in noise, Eq.~(\ref{eq-bit-wise-s4}) becomes
\begin{align}\label{eq-bit-wise-s6}
&\mathrm{P}(f^n(x') \geq f^n(x)) \geq p^2\left(1-\frac{1}{n}\right)^{2(k+1)}\geq p^2\left(1-\frac{k+1}{n}\right)^{2}.
\end{align}
By applying the above two inequalities to Eq.~(\ref{eq-drift-3}), we have
$$\mathrm{E}^- \geq \frac{p^2}{en} \left(\sum^k_{j=1} \frac{j(n-j+1)}{3n^2}+\frac{(n-i-k)(n-1-k)^2}{n^2}\right)=\Omega(1),$$
where the equality is by $p=\Omega(1)$. Thus, $\mathrm{E}^+-\mathrm{E}^-=-\Omega(1)$. That is, condition (1) of Theorem~\ref{simplified-drift} holds.

Since it is necessary to flip at least $j$ bits of $x$, we have
$$
\mathrm{P}(|X_{t+1}-X_{t}|\geq j \mid X_t \geq 1) \leq \binom{n}{j}\frac{1}{n^j} \leq \frac{1}{j!}\leq 2 \cdot \frac{1}{2^j},
$$
which implies that condition (2) of Theorem~\ref{simplified-drift} holds with $\delta=1$ and $r(l)=2$. Note that $l=n^{1/2}$. Thus, by Theorem~\ref{simplified-drift}, the expected running time is exponential.
\end{myproof}

\subsection{Bit-wise Noise $(1,q)$}

For bit-wise noise $(1,q)$, the proof idea is similar to that for bit-wise noise $(p,\frac{1}{n})$. The main difference led by the change of noise is the probability of accepting the offspring solution, i.e., $\mathrm{P}(f^n(x') \geq f^n(x))$. We first prove that the expected running time is polynomial when $q =O(\log n/n^3)$.

\begin{theorem}\label{leadingones-noise-1-poly}
For the (1+1)-EA on LeadingOnes under bit-wise noise $(1,q)$, the expected running time is polynomial if $q =O(\log n/n^3)$.
\end{theorem}
\begin{myproof}
The proof is very similar to that of Theorem~\ref{leadingones-noise-2-poly}. The change of noise only affects the probability of accepting the offspring solution in the analysis. For some positive constant $b$, suppose that $q \leq b \log n /n^3$.

For the positive drift $\mathrm{E}^+$, we need to re-analyze $\mathrm{P}(f^n(x') \geq f^n(x))$ (i.e., Eq.~(\ref{eq-bit-wise-s1}) in the proof of Theorem~\ref{leadingones-noise-2-poly}) for the parent $x$ with $\mathrm{LO}(x)=i$ and the offspring $x'$ with $\mathrm{LO}(x') \geq i+1$. By bit-wise noise $(1,q)$, Eqs.~(\refeq{eq-bit-wise-1}) and~(\refeq{eq-bit-wise-2}) change to
\begin{align}
&\mathrm{P}(f^n(x') \leq i-1)=1-(1-q)^i; \;\; \mathrm{P}(f^n(x) \geq i+1)=(1-q)^iq.
\end{align}
Thus, by the union bound, Eq.~(\ref{eq-bit-wise-s1}) becomes
\begin{align}\label{eq-bit-wise-v1}
&\mathrm{P}(f^n(x') \geq f^n(x))\geq 1-(1-(1-q)^i+(1-q)^iq) \\
&=(1-q)^{i+1}\geq 1-q(i+1) \geq 1-\theta,
\end{align}
where the last inequality holds with sufficiently large $n$, since $q=O(\log n/n^3)$ and $\theta \in (0,1)$ is some constant close to 0.

For the negative drift $\mathrm{E}^-$, we need to re-analyze $\mathrm{P}(f^n(x') \geq f^n(x))$ (i.e., Eq.~(\ref{eq-bit-wise-s2}) in the proof of Theorem~\ref{leadingones-noise-2-poly}) for the parent $x$ with $\mathrm{LO}(x)=i$ (where $i \geq 1$) and the offspring $x'$ with $\mathrm{LO}(x') \leq i-1$. By bit-wise noise $(1,q)$, Eqs.~(\refeq{eq-bit-wise-3}) and~(\refeq{eq-bit-wise-4}) change to
\begin{align}
&\mathrm{P}(f^n(x') \geq i)\leq q(1-q)^{i-1},\quad \mathrm{P}(f^n(x) \leq i-1)=1-(1-q)^i.
\end{align}
Thus, by the union bound, Eq.~(\ref{eq-bit-wise-s2}) becomes
\begin{align}\label{eq-bit-wise-v2}
&\mathrm{P}(f^n(x') \geq f^n(x))\leq q(1-q)^{i-1}+1-(1-q)^i\\
&=1-(1-q)^{i-1}(1-2q)\leq 1-(1-(i-1)q)(1-2q)\leq (i+1)q,
\end{align}
where the second inequality is by $(1-q)^{i-1} \geq 1-(i-1)q$ and $1-2q >0$ for $q=O(\log n/n^3)$.

By applying Eq.~(\refeq{eq-bit-wise-v1}) and Eq.~(\refeq{eq-bit-wise-v2}) to $\mathrm{E}^+$ and $\mathrm{E}^-$, respectively, Eq.~(\refeq{eq-sum-1}) changes to
\begin{align}
&\mathrm{E}(V(\xi_t)-V(\xi_{t+1}) \mid \xi_t=x)\geq  \left(1+\frac{c}{n}\right)^i \left(\frac{(1-\theta)c}{3n^2}-\frac{2q(i+1)}{3}\right)\\
& \geq \left(1+\frac{c}{n}\right)^i \left(\frac{2b\log n+1-\theta}{3n^2}-\frac{2b n\log n }{3n^3}\right) \geq \frac{1-\theta}{3n^2}.
\end{align}
That is, the condition of Theorem~\ref{additive-drift} still holds with $\frac{1-\theta}{3n^2}$. Thus, the expected running time is polynomial.\vspace{0.8em}
\end{myproof}

Next we prove that the expected running time is super-polynomial when $q$ is in the range of $\omega(\log n/n^2) \cap o(1/n)$.

\begin{theorem}\label{leadingones-noise-1-superpoly}
For the (1+1)-EA on LeadingOnes under bit-wise noise $(1,q)$, if $q =\omega(\log n/n^2) \cap o(1/n)$, the expected running time is super-polynomial.
\end{theorem}
\begin{myproof}
We use the same analysis procedure as Theorem~\ref{leadingones-noise-2-superpoly}. The only difference is the probability of accepting the offspring solution $x'$ due to the change of noise. For the positive drift, we still have $\mathrm{E}^+ \leq \frac{i}{n}$, since we optimistically assume that $x'$ is always accepted in the proof of Theorem~\ref{leadingones-noise-2-superpoly}.

For the negative drift, we need to re-analyze $\mathrm{P}(f^n(x') \geq f^n(x))$ for the parent solution $x$ with $\mathrm{LO}(x)=k$ and the offspring solution $x'$ with $\mathrm{LO}(x')=j-1$ (where $1 \leq j \leq k+1$). For $j\leq k$, to derive a lower bound on $\mathrm{P}(f^n(x') \geq f^n(x))$, we consider the $j$ cases where $f^n(x)=l$ and $f^n(x') \geq l$ for $0 \leq l \leq j-1$. Since $\mathrm{P}(f^n(x)=l)=(1-q)^lq$ and $\mathrm{P}(f^n(x') \geq l)=(1-q)^l$, Eq.~(\refeq{eq-bit-wise-s3}) changes to
\begin{align}\label{eq-bit-wise-v3}
&\mathrm{P}(f^n(x') \geq f^n(x)) \geq \sum^{j-1}_{l=0} (1-q)^lq\cdot (1-q)^l\geq \frac{1-(1-q)^{2j}}{2}\\
&=\frac{1}{2} (1-q)^{2j}\left(\left(1+\frac{q}{1-q}\right)^{2j}-1\right) \geq (1-q)^{2j} \frac{qj}{1-q} \geq \frac{qj}{2},
\end{align}
where the last inequality is by $(1-q)^{2j} \geq 1-2qj \geq 1/2$ since $q =o(1/n)$. For $j=k+1$ (i.e., $\mathrm{LO}(x')=\mathrm{LO}(x)=k$), we can use the same analysis as Eq.~(\refeq{eq-bit-wise-v1}) to derive a lower bound $1-q(k+1) \geq 1/2$, where the inequality is by $q=o(1/n)$. Thus, Eq.~(\refeq{eq-bit-wise-s4}) also holds here, i.e.,
\begin{align}\label{eq-bit-wise-v4}
&\mathrm{P}(f^n(x') \geq f^n(x)) \geq \frac{1}{2}.
\end{align}
By applying Eqs.~(\refeq{eq-bit-wise-v3}) and~(\refeq{eq-bit-wise-v4}) to $\mathrm{E}^-$, Eq.~(\refeq{eq-drift-3}) changes to
\begin{align}
&\mathrm{E}^- \geq \frac{qk^2}{12n}+\frac{n-i-k}{6n}.
\end{align}

Thus, we have
$$\mathrm{E}(X_t-X_{t+1} \mid X_t=i)=\mathrm{E}^+-\mathrm{E}^-\leq \frac{i}{n}-\frac{qk^2}{12n}-\frac{n-i-k}{6n}.$$
For the upper bound analysis of $\mathrm{P}(X_{t+1} \neq i \mid X_t=i)$ in the proof of Theorem~\ref{leadingones-noise-2-superpoly}, we only need to replace the acceptance probability $p\frac{k+1}{n}$ in the case of $\mathrm{LO}(x') < \mathrm{LO}(x)$ with $(k+1)q$ (i.e., Eq.~(\refeq{eq-bit-wise-v2})). Thus, Eq.~(\ref{eq-mid1}) changes to
$$\mathrm{P}(X_{t+1} \neq i \mid X_t=i) \leq (k+1)q+\frac{n-k}{n} \leq nq+\frac{n-k}{n}.
$$
To compare $\mathrm{E}(X_t-X_{t+1} \mid X_t=i)$ with $\mathrm{P}(X_{t+1} \neq i \mid X_t=i)$, we consider two cases: $k < n-n^2q$ and $k \geq n-n^2q$. By using $q = \omega(\log n /n^2)$ and applying the same analysis procedure as Eqs.~(\ref{eq-mid2}) and~(\ref{eq-mid3}), we can derive that condition (1) of Theorem~\ref{simplified-drift-selfloops} holds with $\epsilon =\frac{1}{192}$.

For the lower bound analysis of $\mathrm{P}(X_{t+1} \neq i \mid X_t=i)$, by applying Eqs.~(\refeq{eq-bit-wise-v3}) and~(\refeq{eq-bit-wise-v4}), Eq.~(\ref{eq-mid4}) changes to
$$
\mathrm{P}(X_{t+1} \neq i \mid X_t=i) \geq \frac{qk(k+1)}{12n}+\frac{n-k}{6n}.
$$
For the analysis of $|X_{t+1}-X_{t}|\geq j$, by replacing the acceptance probability $p\frac{k+1}{n}$ in the case of $\mathrm{LO}(x') < \mathrm{LO}(x)$ with $(k+1)q$, Eq.~(\ref{eq-mid5}) changes to
\begin{align}
\mathrm{P}(|X_{t+1}-X_{t}|\geq j \mid X_t =i) &\leq \frac{qk(k+1)}{n} \cdot \frac{4}{2^j}+\frac{n-k}{n}\cdot \frac{2}{2^j} \\
&\leq \left(\frac{qk(k+1)}{12n}+\frac{n-k}{6n}\right)\cdot \frac{48}{2^j}.
\end{align}
That is, condition (2) of Theorem~\ref{simplified-drift-selfloops} holds with $\delta=1, r(l)=48$. Thus, the expected running time is super-polynomial.\vspace{0.8em}
\end{myproof}

For $q=\Omega(1/n)$, we prove a stronger result that the expected running time is exponentially lower bounded.

\begin{theorem}\label{leadingones-noise-1-exp}
For the (1+1)-EA on LeadingOnes under bit-wise noise $(1,q)$, the expected running time is exponential if $q=\Omega(1/n)$.
\end{theorem}
\begin{myproof}
We use Theorem~\ref{simplified-drift} to prove it. Let $X_t=i$ be the number of 0-bits of the solution $x$ after $t$ iterations of the (1+1)-EA. We consider the interval $i \in [0,n^{1/2}]$. To analyze the drift $\mathrm{E}(X_t-X_{t+1} \mid X_t=i)$, we use the same analysis procedure as the proof of Theorem~\ref{leadingones-noise-2-superpoly}.

We first consider $q = \Omega(1/n) \cap o(1)$. We need to analyze the probability $\mathrm{P}(f^n(x') \geq f^n(x))$, where the offspring solution $x'$ is generated by flipping only one 1-bit of $x$. Let $\mathrm{LO}(x)=k$. For the case where the $j$-th (where $1 \leq j \leq k$) leading 1-bit is flipped, as the analysis of Eq.~(\ref{eq-bit-wise-v3}), we get
$$\mathrm{P}(f^n(x') \geq f^n(x)) \geq \frac{1-(1-q)^{2j}}{2} \geq (1-q)^{2j} \frac{qj}{1-q}.
$$
If $(1-q)^{2j}<\frac{1}{2}$, $\frac{1-(1-q)^{2j}}{2} \geq \frac{1}{4}$; otherwise, $(1-q)^{2j} \frac{qj}{1-q} \geq \frac{qj}{2}$. Thus, we have
$$\mathrm{P}(f^n(x') \geq f^n(x)) \geq \min\{1/4,qj/2\}.
$$
For the case that flips one non-leading 1-bit (i.e., $\mathrm{LO}(x')=\mathrm{LO}(x)=k$), to derive a lower bound on $\mathrm{P}(f^n(x') \geq f^n(x))$, we consider $f^{n}(x)=l$ and $f^{n}(x') \geq l$ for $0 \leq l \leq k$. Thus,
\begin{equation}
\begin{aligned}
&\mathrm{P}(f^n(x') \geq f^n(x)) \geq \sum^{k-1}_{l=0} (1-q)^lq\cdot (1-q)^l+(1-q)^{k+1} \cdot (1-q)^{k}\\
& \geq \frac{1-(1-q)^{2k}}{2}+(1-q)^{2k+1}=\frac{1}{2}+(1-q)^{2k}\left(\frac{1}{2}-q\right)\geq \frac{1}{2},
\end{aligned}
\end{equation}
where the last inequality is by $q=o(1)$. By applying the above two inequalities to Eq.~(\ref{eq-drift-3}), we get
\begin{align}
&\mathrm{E}^- \geq \frac{1}{en} \left(\sum^{k}_{j=1}\min\left\{\frac{1}{4},\frac{qj}{2}\right\}+\frac{n-i-k}{2}\right).
\end{align}
If $k \geq \frac{n}{2}$, $\sum^{k}_{j=1}\min\{\frac{1}{4},\frac{qj}{2}\}=\Omega(n)$ since $q=\Omega(1/n)$. If $k <\frac{n}{2}$, $\frac{n-i-k}{2}=\Omega(n)$ since $i \leq \sqrt{n}$. Thus, $\mathrm{E}^{-}=\Omega(1)$.

For $q=\Omega(1)$, we use the trivial lower bound $q$ for the probability of accepting the offspring solution $x'$, since it is sufficient to flip the first leading 1-bit of $x$ by noise. Then,
$$
\mathrm{E}^- \geq \frac{1}{en} (k q +(n-i-k) q)=\frac{(n-i)q}{en}=\Omega(1).
$$

Thus, for $q=\Omega(1/n)$, we have
$$
\mathrm{E}(X_t-X_{t+1} \mid X_t=i) =\mathrm{E}^+-\mathrm{E}^-\leq \frac{i}{n}-\Omega(1)=-\Omega(1).
$$
That is, condition (1) of Theorem~\ref{simplified-drift} holds. Its condition (2) trivially holds with $\delta=1$ and $r(l)=2$. Thus, the expected running time is exponential.
\end{myproof}

\subsection{One-bit Noise}

For the (1+1)-EA on LeadingOnes under one-bit noise, it has been known that the expected running time is polynomial if $p \leq 1/(6en^2)$ and exponential if $p=1/2$~\cite{giessen2014robustness}. We extend this result by proving in Theorem~\ref{leadingones-one-bit-noise} that the expected running time is polynomial if $p =O(\log n/n^2)$ and super-polynomial if $p=\omega(\log n/n)$. The proof can be accomplished in the same way as that of Theorems~\ref{leadingones-noise-2-poly},~\ref{leadingones-noise-2-superpoly} and~\ref{leadingones-noise-2-exp} for bit-wise noise $(p,\frac{1}{n})$. This is because although the probabilities $\mathrm{P}(f^n(x') \geq f^n(x))$ of accepting the offspring solution are different, their bounds used in the proofs for bit-wise noise $(p,\frac{1}{n})$ still hold for one-bit noise.

\begin{theorem}\label{leadingones-one-bit-noise}
For the (1+1)-EA on LeadingOnes under one-bit noise, the expected running time is polynomial if $p =O(\log n/n^2)$, super-polynomial if $p=\omega(\log n/n) \cap o(1)$ and exponential if $p=\Omega(1)$.
\end{theorem}
\begin{myproof}
We re-analyze $\mathrm{P}(f^n(x') \geq f^n(x))$ for one-bit noise, and show that the bounds on $\mathrm{P}(f^n(x') \geq f^n(x))$ used in the proofs for bit-wise noise $(p,\frac{1}{n})$ still hold for one-bit noise.

For the proof of Theorem~\ref{leadingones-noise-2-poly}, Eqs.~(\ref{eq-bit-wise-1}) and~(\ref{eq-bit-wise-2}) change to
\begin{align}
&\mathrm{P}(f^n(x') \leq i-1)=p\frac{i}{n}, \qquad \mathrm{P}(f^n(x) \geq i+1)=p\frac{1}{n},
\end{align}
and thus Eq.~(\ref{eq-bit-wise-s1}) still holds; Eqs.~(\ref{eq-bit-wise-3}) and~(\ref{eq-bit-wise-4}) change to
\begin{align}
&\mathrm{P}(f^n(x') \geq i)\leq p\frac{1}{n}, \qquad \mathrm{P}(f^n(x) \leq i-1)=p\frac{i}{n},
\end{align}
and thus Eq.~(\ref{eq-bit-wise-s2}) still holds.

For the proof of Theorem~\ref{leadingones-noise-2-superpoly}, Eqs.~(\refeq{eq-bit-wise-5}) and~(\refeq{eq-bit-wise-6}) change to
\begin{align}
& \mathrm{P}(f^n(x') \geq j-1) = 1-p\frac{j-1}{n}, \quad \mathrm{P}(f^n(x) \leq j-1)=p\frac{j}{n},
\end{align}
and thus Eq.~(\ref{eq-bit-wise-s3}) still holds.

For the proof of Theorem~\ref{leadingones-noise-2-exp}, Eq.~(\refeq{eq-bit-wise-s5}) still holds by the above two equalities; Eq.~(\refeq{eq-bit-wise-s6}) still holds since the probability of keeping the first $(k+1)$ bits of a solution unchanged in one-bit noise is $1-p\frac{k+1}{n} \geq p(1-\frac{k+1}{n})$.
\end{myproof}

\subsection{Experiments}\label{sec-exp}

\begin{figure*}[t!]\centering
\begin{minipage}[c]{0.33\linewidth}\centering
        \includegraphics[width=0.9\linewidth, height=0.7\linewidth]{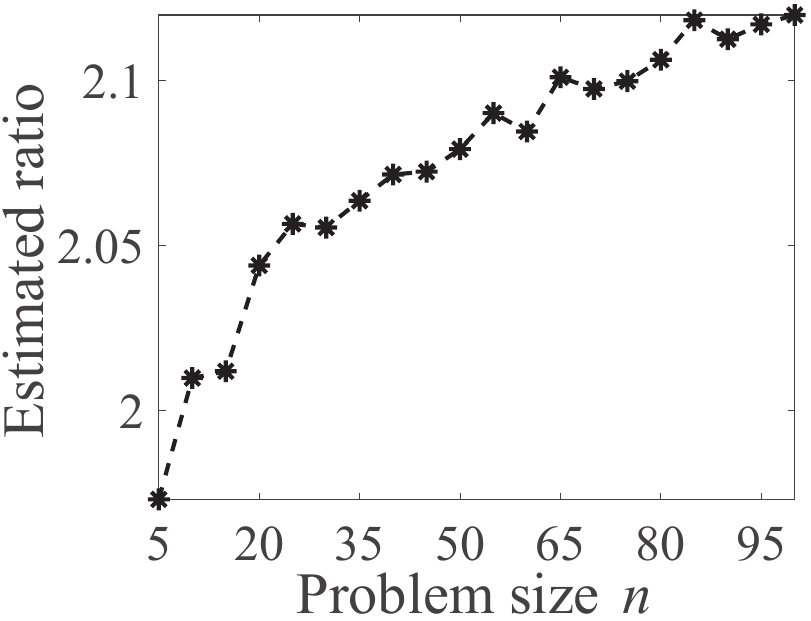}
\end{minipage}
\begin{minipage}[c]{0.33\linewidth}\centering
        \includegraphics[width=0.9\linewidth, height=0.7\linewidth]{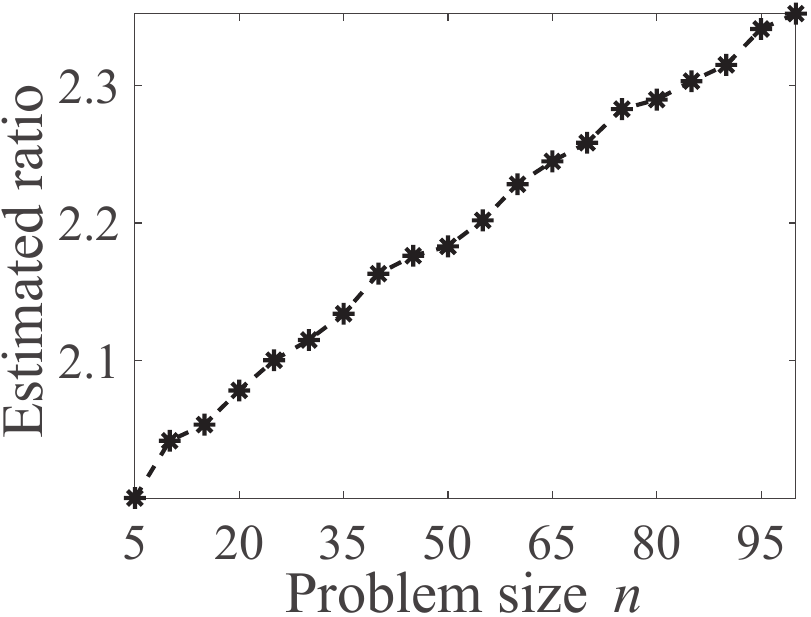}
\end{minipage}
\begin{minipage}[c]{0.33\linewidth}\centering
        \includegraphics[width=0.9\linewidth, height=0.7\linewidth]{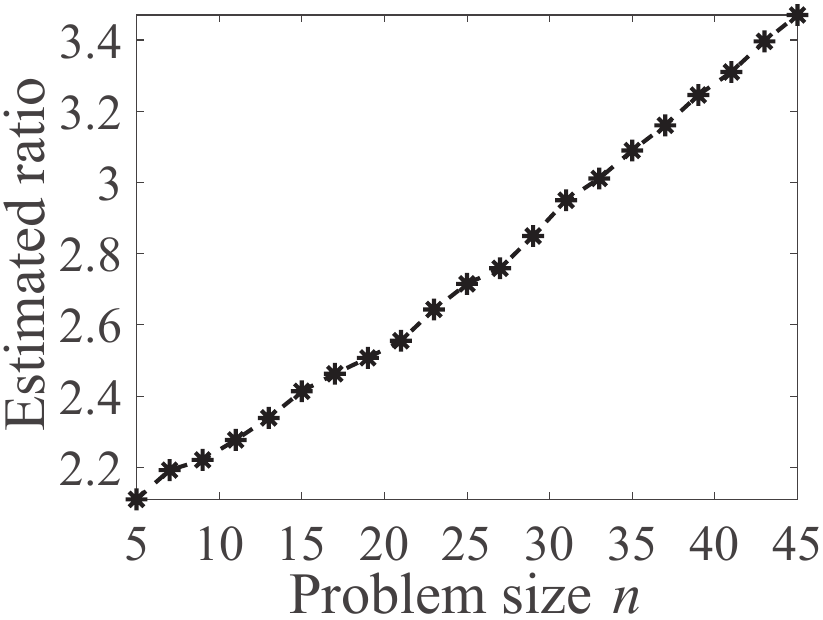}
\end{minipage}\\\vspace{0.3em}
\begin{minipage}[c]{0.33\linewidth}\centering
    \small(a) $p=(\log n/n)^2$
\end{minipage}
\begin{minipage}[c]{0.33\linewidth}\centering
    \small(b) $p=\log n/n^{3/2}$
\end{minipage}
\begin{minipage}[c]{0.33\linewidth}\centering
    \small(c) $p=\log n/n$
\end{minipage}
\caption{Estimated expected running time for the (1+1)-EA on LeadingOnes under bit-wise noise $(p,\frac{1}{n})$, where the $y$-axis is (the logarithm of estimated expected running time) divided by $\log n$.}\label{fig-leadingones-bitwise1}\vspace{-0.7em}
\end{figure*}

\begin{figure*}[t!]\centering
\begin{minipage}[c]{0.33\linewidth}\centering
        \includegraphics[width=0.9\linewidth, height=0.7\linewidth]{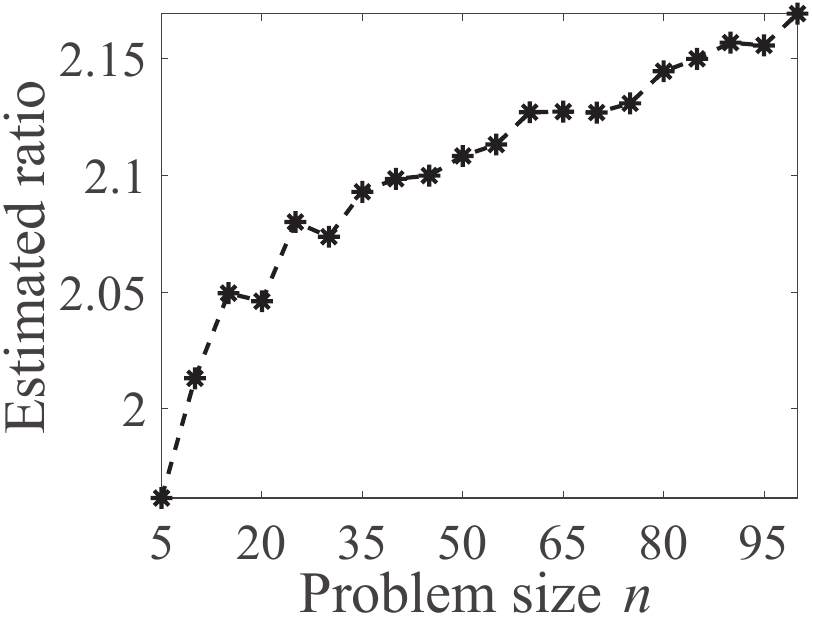}
\end{minipage}
\begin{minipage}[c]{0.33\linewidth}\centering
        \includegraphics[width=0.9\linewidth, height=0.7\linewidth]{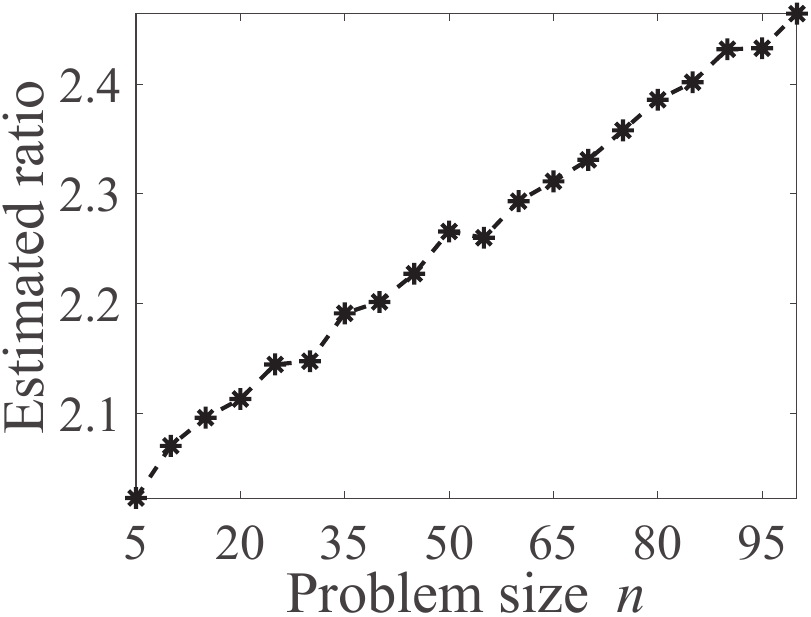}
\end{minipage}
\begin{minipage}[c]{0.33\linewidth}\centering
        \includegraphics[width=0.9\linewidth, height=0.7\linewidth]{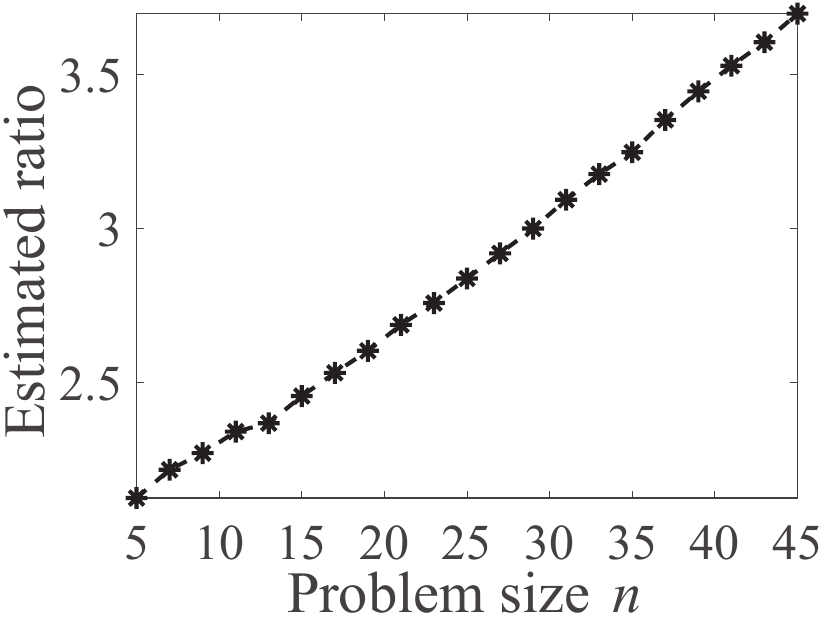}
\end{minipage}\\\vspace{0.3em}
\begin{minipage}[c]{0.33\linewidth}\centering
    \small(a) $q=(\log n)^2/n^3$
\end{minipage}
\begin{minipage}[c]{0.33\linewidth}\centering
    \small(b) $q=\log n/n^{5/2}$
\end{minipage}
\begin{minipage}[c]{0.33\linewidth}\centering
    \small(c) $q=\log n/n^2$
\end{minipage}
\caption{Estimated expected running time for the (1+1)-EA on LeadingOnes under bit-wise noise $(1,q)$, where the $y$-axis is (the logarithm of estimated expected running time) divided by $\log n$.}\label{fig-leadingones-bitwise2}\vspace{-0.7em}
\end{figure*}

\begin{figure*}[t!]\centering
\begin{minipage}[c]{0.33\linewidth}\centering
        \includegraphics[width=0.9\linewidth, height=0.7\linewidth]{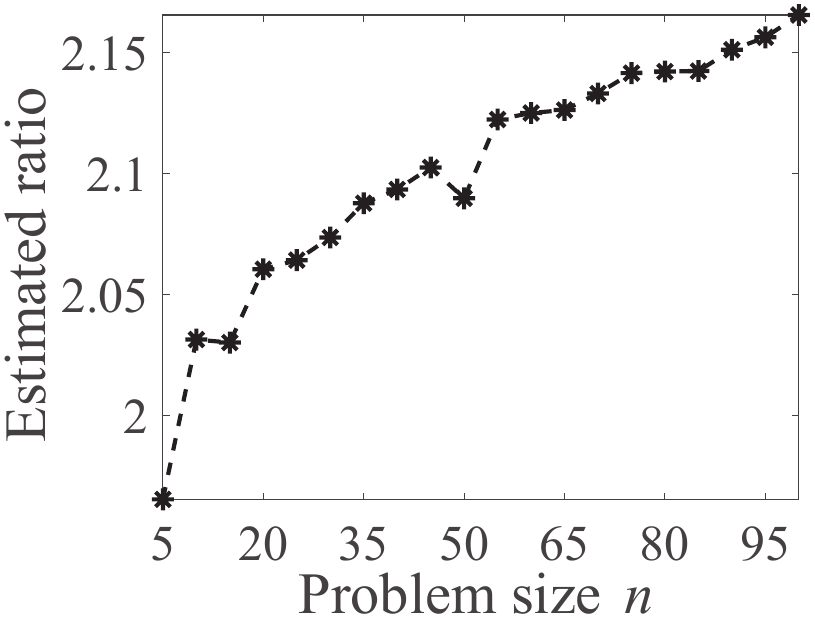}
\end{minipage}
\begin{minipage}[c]{0.33\linewidth}\centering
        \includegraphics[width=0.9\linewidth, height=0.7\linewidth]{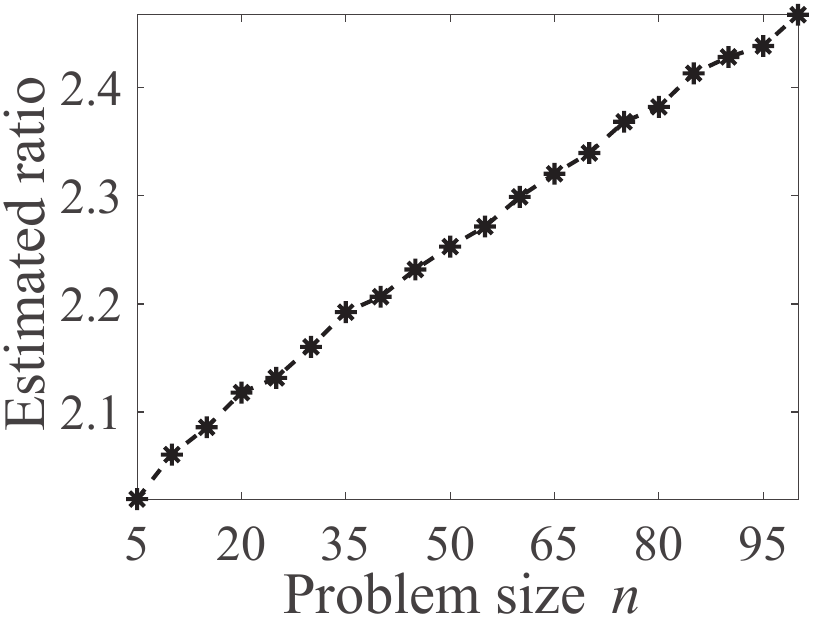}
\end{minipage}
\begin{minipage}[c]{0.33\linewidth}\centering
        \includegraphics[width=0.9\linewidth, height=0.7\linewidth]{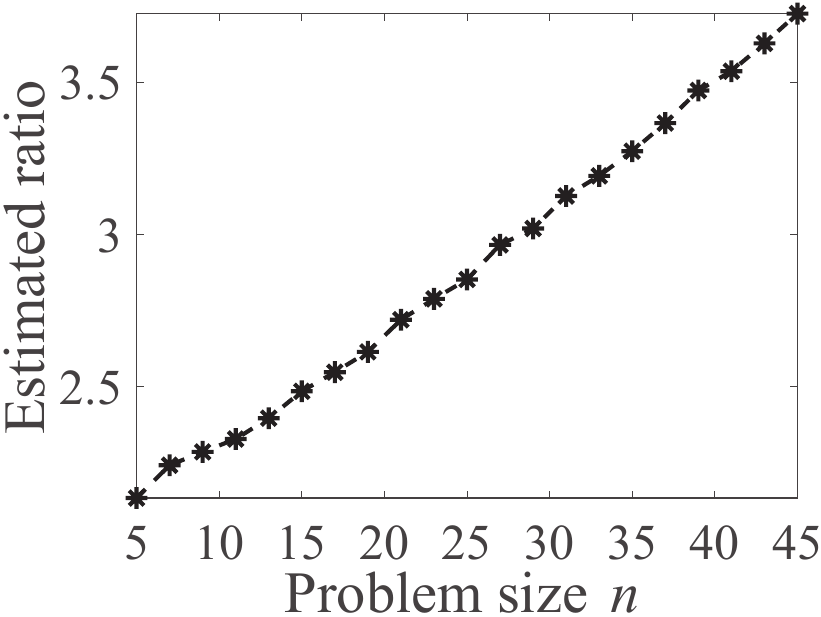}
\end{minipage}\\\vspace{0.3em}
\begin{minipage}[c]{0.33\linewidth}\centering
    \small(a) $p=(\log n/n)^2$
\end{minipage}
\begin{minipage}[c]{0.33\linewidth}\centering
    \small(b) $p=\log n/n^{3/2}$
\end{minipage}
\begin{minipage}[c]{0.33\linewidth}\centering
    \small(c) $p=\log n/n$
\end{minipage}
\caption{Estimated expected running time for the (1+1)-EA on LeadingOnes under one-bit noise, where the $y$-axis is (the logarithm of estimated expected running time) divided by $\log n$.}\label{fig-leadingones-onebit}\vspace{-1em}
\end{figure*}

In the previous three subsections, we have proved that for the (1+1)-EA solving the LeadingOnes problem, if under bit-wise noise $(p,\frac{1}{n})$, the expected running time is polynomial when $p=O(\log n /n^2)$ and super-polynomial when $p=\omega(\log n /n)$; if under bit-wise noise $(1,q)$, the expected running time is polynomial when $q=O(\log n /n^3)$ and super-polynomial when $q=\omega(\log n /n^2)$; if under one-bit noise, the expected running time is polynomial when $p=O(\log n /n^2)$ and super-polynomial when $p=\omega(\log n /n)$. However, the current analysis does not cover all the ranges of $p$ and $q$. We thus have conducted experiments to complement the theoretical results.

For bit-wise noise $(p,\frac{1}{n})$, we do not know whether the running time is polynomial or super-polynomial when $p=\omega(\log n/n^2) \cap O(\log n/n)$. We empirically estimate the expected running time for $p=(\log n/n)^2$, $\log n/n^{3/2}$ and $\log n/n$. On each problem size $n$, we run the (1+1)-EA 1000 times independently. In each run, we record the number of fitness evaluations until an optimal solution w.r.t. the true fitness function is found for the first time. Then the total number of evaluations of the 1000 runs are averaged as the estimation of the expected running time. To show the relationship between the estimated expected running time and the problem size $n$ clearly, we plot the curve of $\log (\text{estimated expected running time})/\log n$, as shown in Figure~\ref{fig-leadingones-bitwise1}. Note that in subfigures (a) and (b), the problem size $n$ is in the range from 5 to 100, while in subfigure (c), $n$ is from 5 to 45. This is because the expected running time with $n > 45$ in subfigure (c) is too large to be estimated. We can observe that all the curves continue to rise as $n$ increases, which suggests that the expected running time for the three tested $p$ values is all $n^{\omega(1)}$, i.e., super-polynomial. For bit-wise noise $(1,q)$ and one-bit noise, we also empirically estimate the expected running time for the values of $q$ and $p$, which are uncovered by our theoretical analysis. The results are plotted in Figures~\ref{fig-leadingones-bitwise2} and~\ref{fig-leadingones-onebit}, respectively, which are similar to that observed for bit-wise noise $(p,\frac{1}{n})$.

Therefore, these empirical results suggest that the expected running time is super-polynomial for the uncovered ranges of $p$ and $q$ in theoretical analysis, and thus the currently derived ranges of $p$ and $q$ allowing a polynomial running time might be tight. The rigorous analysis is not easy. We may need to analyze transition probabilities between fitness levels more precisely, and design an ingenious distance function or use more advanced analysis tools. We leave it as a future work.

Since the theoretical results are all asymptotic, we also empirically compare the expected running time to see when the asymptotic behaviors can be clearly distinguished. For each problem and each kind of noise, we estimate the expected running time for the largest noise level (denoted by $a$) allowing a polynomial running time derived in our theoretical analysis, and then compare it with the estimated expected running time for one relatively smaller noise level $a/\log n$, and two relatively larger noise levels $a\cdot \log n$ and $a \cdot \sqrt{n}$. For example, for the OneMax problem under bit-wise noise $(p,\frac{1}{n})$, the largest noise level allowing a polynomial running time is $O(\log n/n)$; thus we compare the estimated expected running time for $p=1/n$, $\log n/n$, $(\log n)^2/n$ and $\log n/n^{1/2}$. The results are plotted in Figures~\ref{fig-onemax} and~\ref{fig-leadingones}. We can observe that for the OneMax and LeadingOnes problems, the asymptotic polynomial behaviors (i.e., `green $\circ$' and `red $\times$') can be distinguished when the problem size $n$ reaches nearly 200; while the asymptotic behaviors of polynomial (i.e., `green $\circ$' and `red $\times$') and super-polynomial (i.e., `blue $\bigtriangleup$' and `black $\ast$') can be clearly distinguished when $n$ reaches 50 and 100, respectively.

\section{The Robustness of Sampling to Noise}\label{sec-sampling}

From the derived results in the above two sections, we can observe that the (1+1)-EA is efficient for solving OneMax and LeadingOnes only under low noise levels. For example, for the (1+1)-EA solving OneMax under bit-wise noise $(p,\frac{1}{n})$, the optimal solution can be found in polynomial time only when $p=O(\log n/n)$. In this section, we analyze the robustness of the sampling strategy to noise. Sampling as presented in Definition~\ref{sampling} evaluates the fitness of a solution multiple ($m$) times independently and then uses the average to approximate the true fitness. We show that using sampling can significantly increase the largest noise level allowing a polynomial running time. For example, if using sampling with $m=4n^3$, the (1+1)-EA can always solve OneMax under bit-wise noise $(p,\frac{1}{n})$ in polynomial time, regardless of the value of $p$.

\begin{figure*}[t!]\centering
\begin{minipage}[c]{0.33\linewidth}\centering
        \includegraphics[width=0.9\linewidth, height=0.7\linewidth]{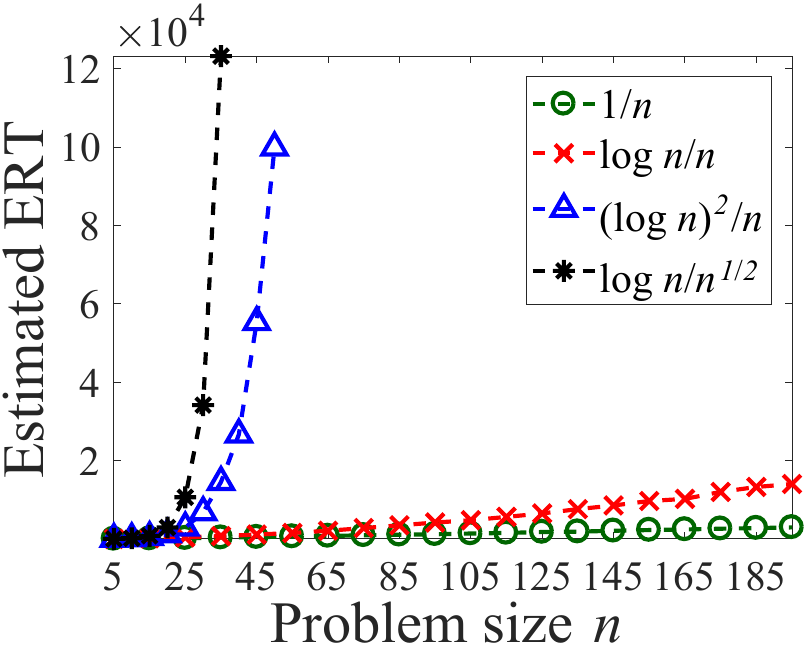}
\end{minipage}
\begin{minipage}[c]{0.33\linewidth}\centering
        \includegraphics[width=0.9\linewidth, height=0.7\linewidth]{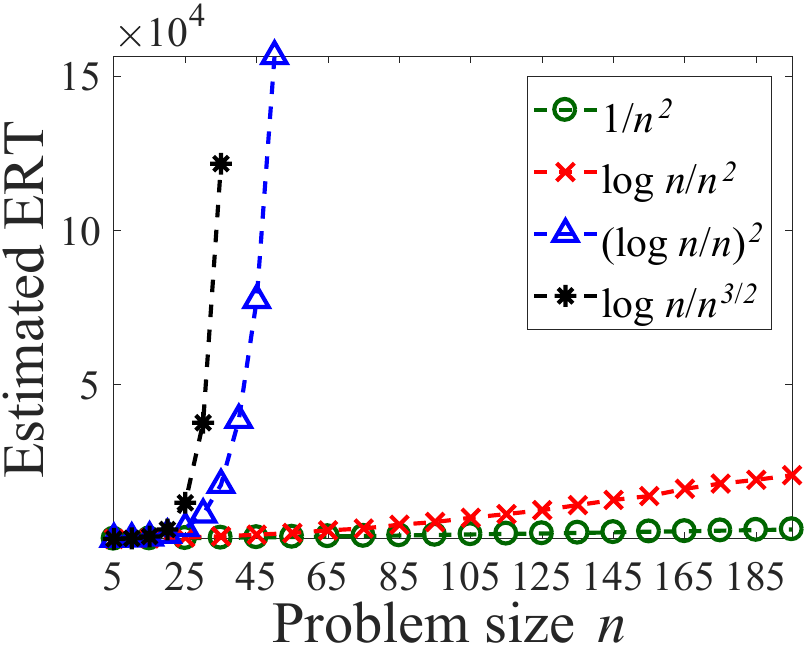}
\end{minipage}
\begin{minipage}[c]{0.33\linewidth}\centering
        \includegraphics[width=0.9\linewidth, height=0.7\linewidth]{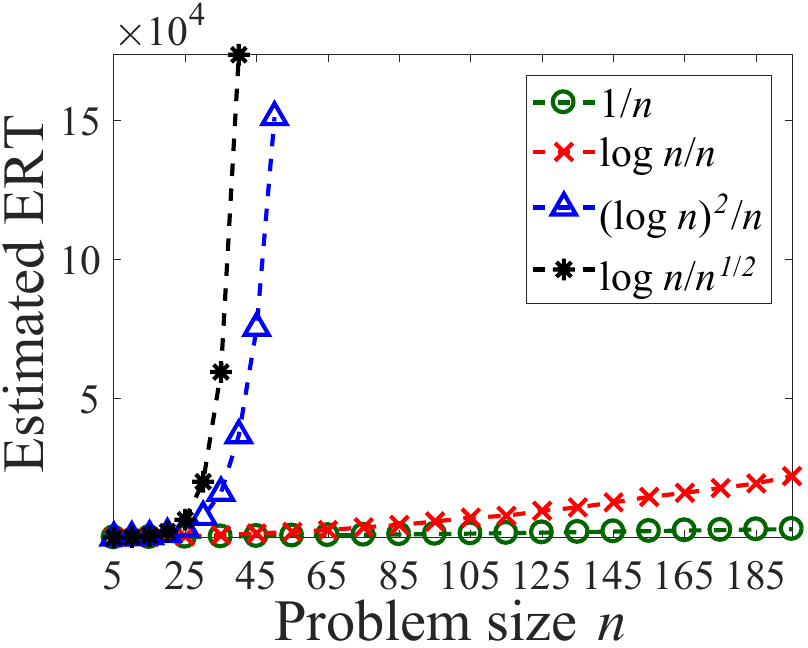}
\end{minipage}\\\vspace{0.3em}
\begin{minipage}[c]{0.33\linewidth}\centering
    \small(a) bit-wise noise $(p,\frac{1}{n})$
\end{minipage}
\begin{minipage}[c]{0.33\linewidth}\centering
    \small(b) bit-wise noise $(1,q)$
\end{minipage}
\begin{minipage}[c]{0.33\linewidth}\centering
    \small(c) one-bit noise
\end{minipage}
\caption{Estimated expected running time (ERT) for the (1+1)-EA on OneMax under noise. Note that for the three studied noise models, the largest noise levels allowing a polynomial running time derived in theoretical analysis are $O(\log n/n)$, $O(\log n/n^2)$ and $O(\log n/n)$, respectively.}\label{fig-onemax}\vspace{-0.7em}
\end{figure*}

\begin{figure*}[t!]\centering
\begin{minipage}[c]{0.33\linewidth}\centering
        \includegraphics[width=0.9\linewidth, height=0.7\linewidth]{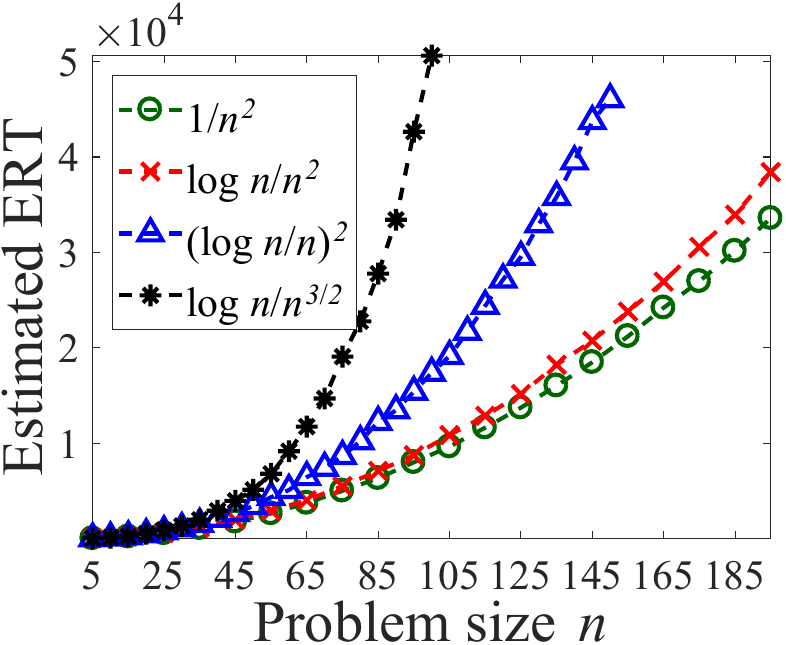}
\end{minipage}
\begin{minipage}[c]{0.33\linewidth}\centering
        \includegraphics[width=0.9\linewidth, height=0.7\linewidth]{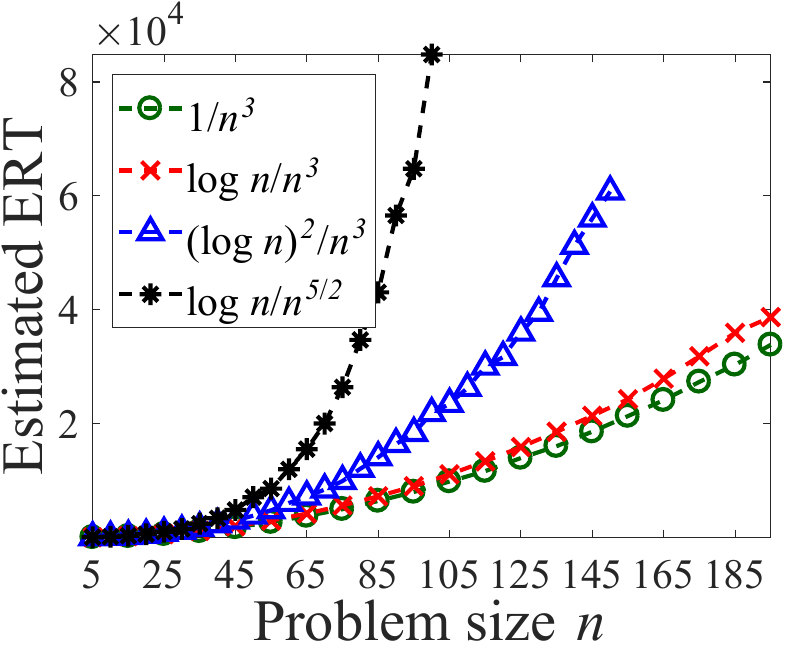}
\end{minipage}
\begin{minipage}[c]{0.33\linewidth}\centering
        \includegraphics[width=0.9\linewidth, height=0.7\linewidth]{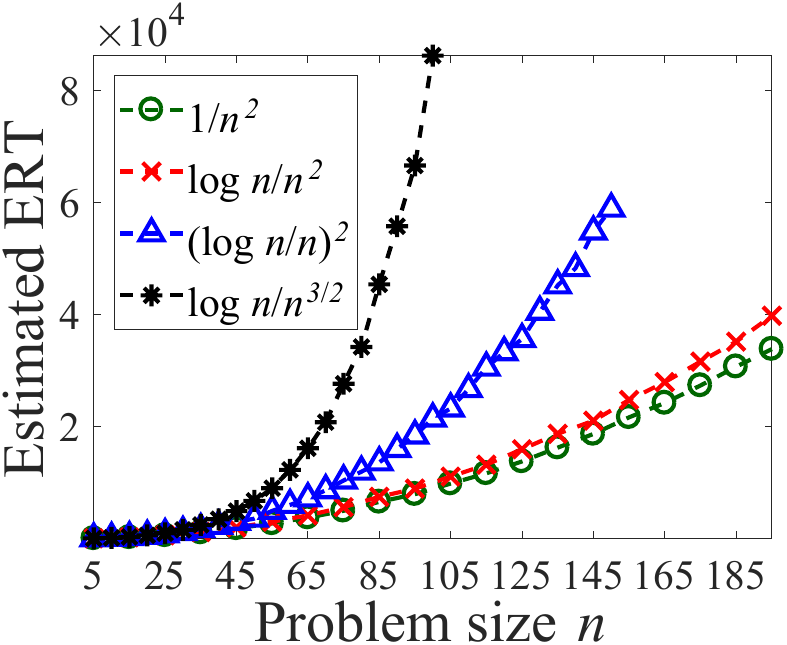}
\end{minipage}\\\vspace{0.3em}
\begin{minipage}[c]{0.33\linewidth}\centering
    \small(a) bit-wise noise $(p,\frac{1}{n})$
\end{minipage}
\begin{minipage}[c]{0.33\linewidth}\centering
    \small(b) bit-wise noise $(1,q)$
\end{minipage}
\begin{minipage}[c]{0.33\linewidth}\centering
    \small(c) one-bit noise
\end{minipage}
\caption{Estimated expected running time (ERT) for the (1+1)-EA on LeadingOnes under noise. Note that for the three studied noise models, the largest noise levels allowing a polynomial running time derived in theoretical analysis are $O(\log n/n^2)$, $O(\log n/n^3)$ and $O(\log n/n^2)$, respectively.}\label{fig-leadingones}\vspace{-1em}
\end{figure*}

\subsection{The OneMax Problem}

We prove in Theorems~\ref{theo-onemax-sampling-upper1} and~\ref{theo-onemax-sampling-upper3} that under bit-wise noise $(p,\frac{1}{n})$ or one-bit noise, the (1+1)-EA can always solve OneMax in polynomial time by using sampling. For bit-wise noise $(1,q)$, the tight range of $q$ allowing a polynomial running time is $1/2-1/n^{O(1)}$, as shown in Theorems~\ref{theo-onemax-sampling-upper2} and~\ref{theo-onemax-sampling-lower}. Let $x^k$ denote any solution with $k$ number of 1-bits, and $f^n(x^{k})$ denote its noisy objective value. For proving polynomial upper bounds, we use Lemma~\ref{onemax-noise-upper}, which gives a sufficient condition based on the probability $\mathrm{P}(f^n(x^j) < f^{n}(x^{k+1}))$ for $j \leq k$. But for the (1+1)-EA using sampling, the probability changes to be $\mathrm{P}(\hat{f}(x^j) < \hat{f}(x^{k+1}))$, where $\hat{f}(x^j)=\frac{1}{m}\sum^{m}_{i=1}f^n_i(x^j)$ as shown in Definition~\ref{sampling}. Lemma~\ref{onemax-noise-upper} requires a lower bound on $\mathrm{P}(\hat{f}(x^j) < \hat{f}(x^{k+1}))$. Our proof idea as presented in Lemma~\ref{onemax-upper-midtheo} is to derive a lower bound on the expectation of $f^{n}(x^{k+1})-f^n(x^j)$ and then apply Chebyshev's inequality. We will directly use Lemma~\ref{onemax-upper-midtheo} in the following proofs. For proving super-polynomial lower bounds, we use Lemma~\ref{onemax-noise-lower} by replacing $\mathrm{P}(f^n(x^k) < f^{n}(x^{k+1}))$ with $\mathrm{P}(\hat{f}(x^k) < \hat{f}(x^{k+1}))$. Let $poly(n)$ indicate any polynomial of $n$.

\begin{lemma}\label{onemax-upper-midtheo}
  Suppose there exists a real number $\delta > 0$ such that
  \begin{equation}
  \forall j\leq k<n:
  \mathrm{E}(f^n(x^{k+1})-f^n(x^j))\ge \delta,
  \end{equation}
then the (1+1)-EA using sampling with $m=n^3/\delta^2$ needs polynomial number of iterations in expectation for solving noisy OneMax.
\end{lemma}
\begin{myproof}
We use Lemma~\ref{onemax-noise-upper} to prove it. For any $j\leq k<n$, let $Y_{k,j}=f^n(x^{k+1})-f^n(x^j)$ and $\hat{Y}_{k,j}=\hat{f}(x^{k+1})-\hat{f}(x^j)$. We then need to analyze the probability $\mathrm{P}(\hat{f}(x^j)<\hat{f}(x^{k+1}))=\mathrm{P}(\hat{Y}_{k,j}>0)$.

Denote the expectation $\mathrm{E}(Y_{k,j})$ as $\mu_{k,j}$ and the variance $\mathrm{Var}(Y_{k,j})$ as $\sigma^2_{k,j}$. It is easy to verify that $\mathrm{E}(\hat{Y}_{k,j})=\mu_{k,j}$ and $\mathrm{Var}(\hat{Y}_{k,j})=\sigma^2_{k,j}/m$. By Chebyshev's inequality, we have
\begin{align}
\mathrm{P}(\hat{Y}_{k,j} \leq 0) \leq \mathrm{P}(|\hat{Y}_{k,j} -\mu_{k,j}|\geq \mu_{k,j}/2)
\leq 4\sigma^2_{k,j}/(m\mu^2_{k,j}).
\end{align}
Since $\mu_{k,j}\geq \delta>0$, $\sigma^2_{k,j}=\mathrm{E}(Y^2_{k,j})-\mu^2_{k,j} \le n^2$ and $m=n^3/\delta^2$, we have
\begin{align}
\mathrm{P}(\hat{Y}_{k,j} \leq 0) \le 4/n \leq \log n/(15n),
\end{align}
where the last inequality holds with sufficiently large $n$. Let $l=\log n$. Then, $\mathrm{P}(\hat{Y}_{k,j} > 0) \geq 1-\frac{\log n}{15n} > 1- \frac{l}{n}$. Let $c=\frac{1}{15}$. For $k<n-l, \mathrm{P}(\hat{Y}_{k,j} >0) \geq  1-c\frac{n-k}{n}$. Thus, the condition of Lemma~\ref{onemax-noise-upper} (i.e., Eq.~(\ref{eq-upper-cond})) holds. We then get that the expected number of iterations is $O(n\log n)+n2^{O(\log n)}=n^{O(1)}$, i.e., polynomial.\vspace{0.8em}
\end{myproof}

For bit-wise noise $(p,\frac{1}{n})$, we apply Lemma~\ref{onemax-upper-midtheo} to prove that the (1+1)-EA using sampling with $m=4n^3$ can always solve OneMax in polynomial time, regardless of the value of $p$.

\begin{theorem}\label{theo-onemax-sampling-upper1}
For the (1+1)-EA on OneMax under bit-wise noise $(p,\frac{1}{n})$, if using sampling with $m=4n^3$, the expected running time is polynomial.
\end{theorem}
\begin{myproof}
We use Lemma~\ref{onemax-upper-midtheo} to prove it. Since $\mathrm{E}(f^n(x^j))=j(1-\frac{p}{n})+(n-j)\frac{p}{n}=(1-\frac{2p}{n})j+p$, we have, for any $j \leq k < n$,
$$
\mathrm{E}(f^n(x^{k+1})-f^n(x^j))=\left(1-\frac{2p}{n}\right)(k+1-j)\geq 1-\frac{2p}{n} \geq 1/2,
$$
where the last inequality holds with $n \geq 4$. Thus, by Lemma~\ref{onemax-upper-midtheo}, we get that the expected number of iterations of the (1+1)-EA using sampling with $m=4n^3$ is polynomial. Since each iteration takes $2m=8n^3$ number of fitness evaluations, the expected running time is also polynomial.\vspace{0.8em}
\end{myproof}

For bit-wise noise $(1,q)$, we prove in the following two theorems that by using sampling, the tight range of $q$ allowing a polynomial running time is $1/2-1/n^{O(1)}$.

\begin{theorem}\label{theo-onemax-sampling-upper2}
For the (1+1)-EA on OneMax under bit-wise noise $(1,q)$ with $q= 1/2-1/n^{O(1)}$, if using sampling, there exists some $m = O(poly(n))$ such that the expected running time is polynomial.
\end{theorem}
\begin{myproof}
We use Lemma~\ref{onemax-upper-midtheo} to prove it. Since $q= 1/2-1/n^{O(1)}$, there exists a positive constant $c$ such that $q\leq 1/2-1/n^c$. It is easy to verify that $\mathrm{E}(f^n(x^j))=j(1-q)+(n-j)q = (1-2q)j+nq$. Thus, for any $j \leq k < n$,
$$
\mathrm{E}(f^n(x^{k+1})-f^n(x^j))=(1-2q)(k+1-j)\geq 1-2q \geq 2/n^{c}.
$$
By Lemma~\ref{onemax-upper-midtheo}, we get that if using sampling with $m=n^{3+2c}/4$, the expected number of iterations is polynomial, and then the expected running time is polynomial. Thus, the theorem holds.
\end{myproof}

\begin{theorem}\label{theo-onemax-sampling-lower}
For the (1+1)-EA on OneMax under bit-wise noise $(1,q)$ with $q=1/2-1/n^{\omega(1)}$ or $q\geq 1/2$, if using sampling with any $m=O(poly(n))$, the expected running time is exponential.
\end{theorem}
\begin{myproof}
We use Lemma~\ref{onemax-noise-lower} to prove it. Note that for the (1+1)-EA using sampling, we have to analyze $\mathrm{P}(\hat{f}(x^k)\!<\! \hat{f}(x^{k+1}))$ instead of $\mathrm{P}(f^n(x^k)\!<\! f^{n}(x^{k+1}))$.

Let $Z$ denote a random variable which satisfies that $\mathrm{P}(Z=0)=q$ and $\mathrm{P}(Z=1)=1-q$. In the following proof, each $Z_i$ is an independent random variable, which has the same distribution as $Z$. We have $
f^n(x^k)=\sum^k_{i=1}Z_i+\sum^n_{i=k+1}(1-Z_i)$, and then,
\begin{align}
&f^n(x^{k+1})-f^n(x^k)\\
& =\sum_{i=1}^{k+1}Z_i+\sum_{i=k+2}^n(1-Z_i) -\sum_{i=n+1}^{n+k}Z_i-\sum_{i=n+k+1}^{2n}(1-Z_i) \\
& =\sum_{i=1}^{n+1}Z_i-\sum_{i=n+2}^{2n}Z_i-1.
\end{align}
Since $\hat{f}(x^k)=\frac{1}{m}\sum^{m}_{i=1} f^n_i(x^k)$, which is the average of $m$ independent evaluations, we have
\begin{align}
&m(\hat{f}(x^{k+1})-\hat{f}(x^k))
=\sum^{m-1}_{j=0}\sum_{i=2nj+1}^{2nj+n+1}Z_i-\sum^{m-1}_{j=0}\sum_{i=2nj+n+2}^{2n(j+1)}Z_i-m\\
&=\sum^{m-1}_{j=0}\sum_{i=2nj+1}^{2nj+2}Z_i +\left(\sum^{m-1}_{j=0}\sum_{i=2nj+3}^{2nj+n+1}Z_i-\sum^{m-1}_{j=0}\sum_{i=2nj+n+2}^{2n(j+1)}Z_i\right)-m\\
&=\sum^{m-1}_{j=0}\sum_{i=2nj+1}^{2nj+2}Z_i +Z^*-m,
\end{align}
where $Z^*=\sum^{m-1}_{j=0}\sum_{i=2nj+3}^{2nj+n+1}Z_i-\sum^{m-1}_{j=0}\sum_{i=2nj+n+2}^{2n(j+1)}Z_i$. To make $\hat{f}(x^k) \geq \hat{f}(x^{k+1})$, it is sufficient that $Z^* \leq 0$ and $\sum^{m-1}_{j=0}\sum_{i=2nj+1}^{2nj+2}Z_i \leq m$. That is,
\begin{align}\label{mid-eq3}
\mathrm{P}(\hat{f}(x^k) \geq \hat{f}(x^{k+1})) \geq \mathrm{P}(Z^* \leq 0) \cdot \mathrm{P}\left(\sum^{m-1}_{j=0}\sum_{i=2nj+1}^{2nj+2}Z_i \leq m\right).
\end{align}
Since $Z^*$ is the difference between the sum of the same number of $Z_i$, $Z^*$ has the same distribution as $-Z^*$. Thus, $\mathrm{P}(Z^*\leq 0)+\mathrm{P}(Z^*\geq 0)=\mathrm{P}(Z^*\leq 0)+\mathrm{P}(-Z^*\leq 0)=2\mathrm{P}(Z^*\leq 0)\geq 1$, which implies that
\begin{align}\label{mid-eq1}
\mathrm{P}(Z^*\leq 0)\geq 1/2.
\end{align}
We then investigate $\mathrm{P}(\sum^{m-1}_{j=0}\sum_{i=2nj+1}^{2nj+2}Z_i\leq m)$. Since $\sum^{m-1}_{j=0}\sum_{i=2nj+1}^{2nj+2}Z_i$ is the sum of $2m$ independent random variables which have the same distribution as $Z$, we have
$$
\mathrm{P}\left(\sum^{m-1}_{j=0}\sum_{i=2nj+1}^{2nj+2}Z_i \leq m\right)=\sum_{t=0}^m\binom{2m}{t}(1-q)^tq^{2m-t},
$$
$$
\mathrm{P}\left(\sum^{m-1}_{j=0}\sum_{i=2nj+1}^{2nj+2}Z_i > m\right)\!=\!\!\sum_{t=m+1}^{2m}\!\!\binom{2m}{t}(1-q)^tq^{2m-t}
\!=\!\sum_{t=0}^{m-1}\!\!\binom{2m}{t}(1-q)^{2m-t}q^t.
$$
For any $t<m$, let $r= \frac{(1-q)^tq^{2m-t}}{(1-q)^{2m-t}q^t} =(\frac{q}{1-q})^{2m-2t}$. If $q\geq 1/2$, we have $r\geq1$. If $q=1/2-1/n^{\omega(1)}$, we have
\begin{align}
r&\geq \left(\frac{q}{1-q}\right)^{2m}=\left(1-\frac{1-2q}{1-q}\right)^{2m}\\
&\geq \left(1-\frac{4}{n^{\omega(1)}}\right)^{2m}\geq \left(\frac{1}{e}\right)^{2m/(n^{\omega(1)}/4-1)}\geq \frac{1}{e},
\end{align}
where the first inequality is by $q\leq 1/2$, the second inequality is by $1-2q=2/n^{\omega(1)}$ and $1-q \geq 1/2$, and the last is by $m=O(poly(n))$. Thus, $\mathrm{P}(\sum^{m-1}_{j=0}\sum_{i=2nj+1}^{2nj+2}Z_i \leq m)>1/3\cdot \mathrm{P}(\sum^{m-1}_{j=0}\sum_{i=2nj+1}^{2nj+2}Z_i > m)$, which implies that
\begin{align}\label{mid-eq2}
\mathrm{P}\left(\sum^{m-1}_{j=0}\sum_{i=2nj+1}^{2nj+2}Z_i \leq m\right)>1/4.
\end{align}
By applying Eqs.~(\refeq{mid-eq1}) and~(\refeq{mid-eq2}) to Eq.~(\refeq{mid-eq3}), we get
$$
\mathrm{P}(\hat{f}(x^k) \geq \hat{f}(x^{k+1}))\geq  1/8.
$$

Let $c=16$ and $l=n/128$. For any $n-l\leq k<n$,
$$
\mathrm{P}(\hat{f}(x^{k})< \hat{f}(x^{k+1}))=1-\mathrm{P}(\hat{f}(x^{k})\geq \hat{f}(x^{k+1})) \leq 1- \frac{cl}{n}\leq 1- \frac{c(n-k)}{n},
$$
i.e., the condition of Lemma~\ref{onemax-noise-lower} holds. Thus, the expected number of iterations is $2^{\Omega(n/128)}$, and the expected running time is exponential.\vspace{0.8em}
\end{myproof}

For one-bit noise, we show that using sampling with $m =4n^3$ is sufficient to make the (1+1)-EA solve OneMax in polynomial time.

\begin{theorem}\label{theo-onemax-sampling-upper3}
For the (1+1)-EA on OneMax under one-bit noise, if using sampling with $m =4n^3$, the expected running time is polynomial.
\end{theorem}
\begin{myproof}
It is easy to verify that the expectation of $f^n(x^j)$ (i.e., $\mathrm{E}(f^n(x^j))$) under one-bit noise is the same as that under bit-wise noise $(p,\frac{1}{n})$. Thus, the proof can be finished in the same way as that of Theorem~\ref{theo-onemax-sampling-upper1}.\vspace{0.8em}
\end{myproof}

From the above analysis, we can intuitively explain why sampling is always effective for bit-wise noise $(p,\frac{1}{n})$ and one-bit noise, while it fails for bit-wise noise $(1,q)$ when $q=1/2-1/n^{\omega(1)}$ or $q\geq 1/2$. For two solutions $x$ and $y$ with $f(x) > f(y)$, if under bit-wise noise $(p,\frac{1}{n})$ and one-bit noise, the noisy fitness $f^n(x)$ is larger than $f^n(y)$ in expectation, and using sampling will increase this trend and make the probability of accepting the true worse solution $y$ sufficiently small. If under bit-wise noise $(1,q)$, when $q=1/2-1/n^{\omega(1)}$, although the noisy fitness $f^n(x)$ is still larger in expectation, the gap is very small (in the order of $1/n^{\omega(1)}$) and a polynomial sample size is not sufficient to make the probability of accepting the true worse solution $y$ small enough; when $q\geq 1/2$, the noisy fitness $f^n(x)$ is smaller in expectation, and using sampling will increase this trend and it obviously does not work.

\subsection{The LeadingOnes Problem}

The bit-wise noise $(p,\frac{1}{n})$ model is first considered. We prove in Theorem~\ref{theo-leadingones-sampling-upper1} that the (1+1)-EA using sampling can solve the LeadingOnes problem in polynomial time, regardless of the value of $p$. The proof idea is similar to that of Theorem~\ref{leadingones-noise-2-poly}. The main difference is the probability of accepting the offspring solution $x'$, which is changed from $\mathrm{P}(f^n(x')\geq f^n(x))$ to $\mathrm{P}(\hat{f}(x')\geq \hat{f}(x))$ due to sampling. Lemma~\ref{lemma-leadingones-sampling-prob} gives some bounds on this probability, which will be used in the proof of Theorem~\ref{theo-leadingones-sampling-upper1}.

\begin{lemma}\label{lemma-leadingones-sampling-prob}
For the LeadingOnes problem under bit-wise noise $(p,\frac{1}{n})$, if using sampling with $m=144n^6$, it holds that
\begin{enumerate}[{(1)}]
  \item for any $x$ with $\mathrm{LO}(x)=i<n$ and $y$ with $\mathrm{LO}(y) \leq i-2$ or $\mathrm{LO}(y)=i-1 \wedge y_{i+1}=0$, $\mathrm{P}(\hat{f}(x) \leq \hat{f}(y))\leq 1/n^2$.
  \item for any $y$ with $\mathrm{LO}(y) <n$, $\mathrm{P}(\hat{f}(1^n) \leq \hat{f}(y))\leq 1/(4n^4)$.
\end{enumerate}
\end{lemma}
\begin{myproof}
The proof is finished by deriving a lower bound on the expectation of $f^n(x)-f^n(y)$ (which is equal to the expectation of $\hat{f}(x)-\hat{f}(y)$) and then applying Chebyshev's inequality. We first consider case~(1). For any $x$ with $\mathrm{LO}(x)=i<n$,
\begin{align}\label{LO,(p,1/n):E_uppercase1}
\mathrm{E}(f^n(x))& \geq (1-p)\cdot i+
\sum_{j=1}^i p\left(1-\frac{1}{n}\right)^{j-1}\frac{1}{n}\cdot (j-1)\nonumber\\
& \quad+p\left(1-\frac{1}{n}\right)^{i}\frac{1}{n}\cdot (i+1)+
p\left(1-\frac{1}{n}\right)^{i+1}\cdot i,\\
\mathrm{E}(f^n(x)) &\leq (1-p)\cdot i+
\sum_{j=1}^{i}p\left(1-\frac{1}{n}\right)^{j-1}\frac{1}{n}\cdot (j-1) \nonumber\\
& \quad+p\left(1-\frac{1}{n}\right)^{i}\frac{1}{n}\cdot n
+p\left(1-\frac{1}{n}\right)^{i+1}\cdot i.
\end{align}
Note that when flipping the first 0-bit of $x$ and keeping the $i$ leading 1-bits unchanged, the fitness is at least $i+1$ and at most $n$. Then for any $1\leq i<n$, we have
\begin{align}\label{LO,(p,1/n):Diff_E}
&\mathrm{E}(f^n(x)-f^n(y)\mid \mathrm{LO}(x)=i \wedge \mathrm{LO}(y)=i-1)\\
& \geq 1-p+p\left(1-\frac{1}{n}\right)^{i-1}\frac{1}{n}\cdot (i-1) +\frac{p}{n}\left(1-\frac{1}{n}\right)^{i-1}\cdot \left(\left(1-\frac{1}{n}\right)(i+1)-n\right)\nonumber \\
& \quad+p\left(1-\frac{1}{n}\right)^{i}\left(\left(1-\frac{1}{n}\right)i-(i-1)\right) \nonumber \\
& =1-p+p\left(1-\frac{1}{n}\right)^{i-1}\left(\frac{i-1}{n}-\frac{1}{n^2}\right) \geq
\left\{\begin{array}{lcl}
-1/n^2, \; & \text{if} \;\;\; i=1 \\
1/(6n), \; & \text{if} \;\;\; i\geq 2\\
\end{array} \right..
\end{align}
Thus, for any $x$ with $\mathrm{LO}(x)=i$ and $y$ with $\mathrm{LO}(y) \leq i-2$ (where $2 \leq i<n$), letting $z$ be any solution with $\mathrm{LO}(z)=\mathrm{LO}(y)+1$, we have
\begin{align}\label{lower-bound-case1}
\mathrm{E}(f^n(x)-f^n(y))&=\mathrm{E}(f^n(x)-f^n(z))+\mathrm{E}(f^n(z)-f^n(y))\\
&\geq \frac{1}{6n} \cdot (i-\mathrm{LO}(y)-1)-\frac{1}{n^2} \geq \frac{1}{6n} -\frac{1}{n^2} \geq \frac{1}{12n},
\end{align}
where the first inequality is by repeatedly applying Eq.~(\refeq{LO,(p,1/n):Diff_E}), the second inequality is by $\mathrm{LO}(y)\leq i-2$, and the last holds with $n \geq 12$.

When $\mathrm{LO}(y)=i-1 \wedge y_{i+1}=0$, we have to re-analyze Eq.~(\ref{LO,(p,1/n):Diff_E}) to derive a tighter lower bound, because applying Eq.~(\ref{LO,(p,1/n):Diff_E}) directly will lead to a negative lower bound for $i=1$. We first derive a tighter upper bound on $\mathrm{E}(f^n(y))$. When flipping the $i$-th bit of $y$ and keeping the $i-1$ leading 1-bits unchanged, we can now further consider the flipping of the $(i+1)$-th bit since we know that $y_{i+1}=0$, rather than directly using a trivial upper bound $n$ on the noisy fitness. If $y_{i+1}$ is not flipped, $f^n(y)= i$; otherwise, $f^n(y)\leq n$. Thus, we get
\begin{align} \label{LO,(p,1/n):E_uppercase2}
&\mathrm{E}(f^n(y) \mid \mathrm{LO}(y)=i-1 \wedge y_{i+1}=0)\\
& \leq (1-p)\cdot (i-1)+
\sum_{j=1}^{i-1}p\left(1-\frac{1}{n}\right)^{j-1}\frac{1}{n}\cdot (j-1) \nonumber\\
& \quad+p\left(1-\frac{1}{n}\right)^{i-1}\frac{1}{n}\cdot \left(\frac{1}{n}n+\left(1-\frac{1}{n}\right)i\right)
+p\left(1-\frac{1}{n}\right)^{i}\cdot (i-1).
\end{align}
By combining this inequality and the lower bound in Eq.~(\refeq{LO,(p,1/n):E_uppercase1}), we have that, for any $x$ with $\mathrm{LO}(x)=i$ and $y$ with $\mathrm{LO}(y) = i-1 \wedge y_{i+1}=0$ (where $1 \leq i<n$),
\begin{align}\label{lower-bound-case2}
&\mathrm{E}(f^n(x)-f^n(y)) \\
&\geq 1-p+p\left(1-\frac{1}{n}\right)^{i-1}\frac{1}{n}\cdot (i-1) +p\left(1-\frac{1}{n}\right)^{i}\left(\left(1-\frac{1}{n}\right)i-(i-1)\right)\\
& \quad+\frac{p}{n}\left(1-\frac{1}{n}\right)^{i-1}\cdot \left(\left(1-\frac{1}{n}\right)(i+1)-\left(i+1-\frac{i}{n}\right)
\right) \\
& =1-p+p\left(1-\frac{1}{n}\right)^{i-1}\left(1-\frac{2}{n}+\frac{i-1}{n^2}\right) \\
& \geq 1-p+p\cdot \frac{1}{e}\cdot \frac{1}{2} \geq 1-\frac{5}{6}p \geq \frac{1}{6},
\end{align}
where the second inequality holds with $n \geq 4$.

According to Eqs.~(\refeq{lower-bound-case1}) and~(\refeq{lower-bound-case2}), we have a unified lower bound $1/(12n)$ on $\mathrm{E}(f^n(x)-f^n(y))$. Denote $\mathrm{E}(f^n(x)-f^n(y))$ as $\mu$ and $\mathrm{Var}(f^n(x)-f^n(y))$ as $\sigma^2$. We have $\mu \geq 1/(12n)$ and $\sigma^2 \leq n^2$ since $|f^n(x)-f^n(y)| \leq n$. As $\hat{f}(x)$ is the average of $m=144n^6$ independent evaluations, it is easy to verify that $\mathrm{E}(\hat{f}(x)-\hat{f}(y))=\mu$ and $\mathrm{Var}(\hat{f}(x)-\hat{f}(y))=\sigma^2/m$. By Chebyshev's inequality,
\begin{align}\label{LO,(p,1/n):Chebyshev}
\mathrm{P}(\hat{f}(x) \leq \hat{f}(y))&\leq \mathrm{P}(|(\hat{f}(x)-\hat{f}(y))-\mu|\geq \mu)\leq \sigma^2/(m\mu^2)\leq 1/n^2.
\end{align}
Thus, case~(1) holds.

For case~(2), we first analyze $\mathrm{E}(f^n(1^n)-f^n(1^{n-1}0))$. The expectation on $f^n(1^n)$ can be easily calculated as follows:
\begin{align} \label{LO,(p,1/n):E-n}
\mathrm{E}(f^n(1^n))=(1-p)\cdot n+\sum_{j=1}^{n}p\left(1-\frac{1}{n}\right)^{j-1}\frac{1}{n}\cdot(j-1)
+p\left(1-\frac{1}{n}\right)^n\cdot n.
\end{align}
Combining this equality with the upper bound in Eq.~(\ref{LO,(p,1/n):E_uppercase1}), we get
\begin{align}\label{LO,(p,1/n):Diff_Es}
&\mathrm{E}(f^n(1^n)-f^n(1^{n-1}0))\\
& \geq 1-p+p\left(1-\frac{1}{n}\right)^{n-1}\frac{1}{n}\cdot (n-1) +p\left(1-\frac{1}{n}\right)^{n}-p\left(1-\frac{1}{n}\right)^{n-1} \nonumber\\
& =1-p+p\left(1-\frac{1}{n}\right)^{n-1}\left(1-\frac{2}{n}\right) \geq \frac{1}{6}.
\end{align}
Then, for any $y$ with $\mathrm{LO}(y)<n$, we have
$$
\mathrm{E}(f^n(1^n)-f^n(y))\!=\!\mathrm{E}(f^n(1^n)-f^n(1^{n-1}0))+\!\!\!
\sum^{n-1}_{k=\mathrm{LO}(y)+1}\!\!\!\!\! \mathrm{E}(f^n(z^{k})-f^n(z^{k-1}))\!\geq\! \frac{1}{6},
$$
where $z^{n-1}=1^{n-1}0$, $z^k \; (\mathrm{LO}(y)<k<n-1)$ denotes one solution $z$ with $\mathrm{LO}(z)=k$, $z^{\mathrm{LO}(y)}=y$, and the inequality is by applying Eqs.~(\refeq{LO,(p,1/n):Diff_Es}) and~(\ref{LO,(p,1/n):Diff_E}). As the analysis for $\mathrm{P}(\hat{f}(x)\leq \hat{f}(y))$ in case~(1) (i.e., Eq.~(\refeq{LO,(p,1/n):Chebyshev})), we can similarly use Chebyshev's inequality to derive that, noting that $\mu \geq 1/6$ here,
$$
\mathrm{P}(\hat{f}(1^{n})\leq \hat{f}(y)) \leq \sigma^2/(m\mu^2) \leq 1/(4n^4).
$$
Thus, case~(2) holds.\vspace{0.8em}
\end{myproof}

The following theorem shows that for bit-wise noise $(p,\frac{1}{n})$, using sampling with $m=144n^6$ is sufficient to make the (1+1)-EA solve LeadingOnes in polynomial time.

\begin{theorem}\label{theo-leadingones-sampling-upper1}
For the (1+1)-EA on LeadingOnes under bit-wise noise $(p,\frac{1}{n})$,if using sampling with $m=144n^6$, the expected running time is polynomial.
\end{theorem}
\begin{myproof}
We use Theorem~\ref{additive-drift} to prove it. We first construct a distance function $V(x)$ as, for any $x$ with $\mathrm{LO}(x)=i$, $$V(x)=\left(1+\frac{c}{n}\right)^n-\left(1+\frac{c}{n}\right)^i,$$ where $c=13$. Then, we investigate $\mathrm{E}(V(\xi_t)-V(\xi_{t+1}) \mid \xi_t=x)$ for any $x$ with $\mathrm{LO}(x)<n$. Assume that currently $\mathrm{LO}(x)=i$, where $0\leq i \leq n-1$. We divide the drift into two parts: positive $\mathrm{E}^+$ and negative $\mathrm{E}^-$. That is, $$\mathrm{E}(V(\xi_t)-V(\xi_{t+1}) \mid \xi_t=x)=\mathrm{E}^+-\mathrm{E}^-.$$

The positive drift $\mathrm{E}^+$ can be expressed as Eq.~(\refeq{eq-positive-drift}), except that $\mathrm{P}(f^n(x')\geq f^n(x))$ changes to $\mathrm{P}(\hat{f}(x')\geq \hat{f}(x))$ due to sampling. To derive a lower bound on $\mathrm{E}^+$, we only consider that the $(i+1)$-th bit of $x$ is flipped and the other bits keep unchanged, the probability of which is $\frac{1}{n}(1-\frac{1}{n})^{n-1}$. The only difference between $x'$ and $x$ is the $(i+1)$-th bit and $\mathrm{LO}(x')>\mathrm{LO}(x)=i$. If $\mathrm{LO}(x')=n$, $\mathrm{P}(\hat{f}(x') \leq \hat{f}(x)) \leq 1/(4n^4)$ by case~(2) of Lemma~\ref{lemma-leadingones-sampling-prob}, and then $\mathrm{P}(\hat{f}(x') \geq \hat{f}(x)) \geq 1-1/(4n^4)$. If $\mathrm{LO}(x')<n$, it must hold that $\mathrm{LO}(x)\leq \mathrm{LO}(x')-1 \wedge x_{\mathrm{LO}(x')+1}=x'_{\mathrm{LO}(x')+1}=0$. By case~(1) of Lemma~\ref{lemma-leadingones-sampling-prob}, $\mathrm{P}(\hat{f}(x') \leq \hat{f}(x)) \leq 1/n^2$, and then $\mathrm{P}(\hat{f}(x') \geq \hat{f}(x)) \geq 1-1/n^2$. Thus, the probability of accepting the offspring solution $x'$ is at least $1/2$. Since $\mathrm{LO}(x')>i$, $V(x)-V(x') \geq (1+\frac{c}{n})^{i+1}-(1+\frac{c}{n})^i=\frac{c}{n}(1+\frac{c}{n})^i$. Then, $\mathrm{E}^+$ can be lower bounded as follows:
\begin{align} \label{LO,(p,1/n):E+}
\mathrm{E}^+ \geq  \left(1-\frac{1}{n}\right)^{n-1}\frac{1}{n}\cdot \frac{1}{2} \cdot \frac{c}{n}\left(1+\frac{c}{n}\right)^i\geq \frac{c}{6n^2}\left(1+\frac{c}{n}\right)^i.
\end{align}

For the negative drift $\mathrm{E}^{-}$, we need to consider $\mathrm{LO}(x')<\mathrm{LO}(x)=i$. Since $V(x')-V(x) \leq V(0^n)-V(x)= (1+\frac{c}{n})^i-1$, Eq.~(\refeq{eq-negative-drift}) becomes
$$\mathrm{E}^-\leq \left(\left(1+\frac{c}{n}\right)^i-1\right)\sum_{x': \mathrm{LO}(x')<i}\mathrm{P}_{mut}(x,x')\cdot \mathrm{P}(\hat{f}(x') \geq \hat{f}(x)).
$$
We further divide $\mathrm{LO}(x')<i$ into two cases. If $\mathrm{LO}(x')\leq i-2$ or $\mathrm{LO}(x')= i-1 \wedge x'_{i+1}=0$, then $\mathrm{P}(\hat{f}(x') \geq \hat{f}(x)) \leq 1/n^2$ by case~(1) of Lemma~\ref{lemma-leadingones-sampling-prob}. If $\mathrm{LO}(x')= i-1 \wedge x'_{i+1}=1$, then $\sum_{x': \mathrm{LO}(x')= i-1 \wedge x'_{i+1}=1}\mathrm{P}_{mut}(x,x') \leq 1/n^2$ since it is necessary to flip the $i$-th and the $(i+1)$-th bits of $x$ in mutation. Then, we get
\begin{align}\label{LO,(p,1/n):E-}
& \mathrm{E}^-\leq \left(\left(1+\frac{c}{n}\right)^i-1\right)\cdot \left(1\cdot \frac{1}{n^2}
 +\frac{1}{n^2}\cdot 1\right) \leq \frac{2}{n^2}\left(1+\frac{c}{n}\right)^i.
\end{align}

By subtracting $\mathrm{E}^-$ from $\mathrm{E}^+$, we have, noting that $c=13$,
\begin{align}\label{LO,(p,1/n):E}
\mathrm{E}(V(\xi_t)-V(\xi_{t+1}) \mid \xi_t=x)
&\ge \left(1+\frac{c}{n}\right)^i \cdot \left(\frac{c}{6n^2}-\frac{2}{n^2}\right) \geq \frac{1}{6n^2}.
\end{align}
Since $V(x) \!\le\! (1\!+\!\frac{13}{n})^n \!\le\! e^{13}=O(1)$, we have $\mathrm{E}(\tau\mid \xi_0)=O(n^{2})$ by Theorem~\ref{additive-drift}. Each iteration of the (1+1)-EA using sampling takes $2m=288n^6$ number of fitness evaluations, thus the expected running time is polynomial.\vspace{0.8em}
\end{myproof}

For bit-wise noise $(1,q)$, we prove in Theorems~\ref{theo-leadingones-sampling-upper2} and~\ref{theo-leadingones-sampling-lower1} that the expected running time is polynomial if and only if $q = O(\log n/n)$. The proof of Theorem~\ref{theo-leadingones-sampling-upper2} is similar to that of Theorem~\ref{theo-leadingones-sampling-upper1}, which considers bit-wise noise $(p,\frac{1}{n})$. The main difference is the probability of accepting the offspring solution $x'$ (i.e., $\mathrm{P}(\hat{f}(x') \!\geq\! \hat{f}(x))$), due to the change of noise. Lemma~\ref{lemma-leadingones-sampling-prob2} gives some bounds on this probability, which will be used in the proof of Theorem~\ref{theo-leadingones-sampling-upper2}.

\begin{lemma}\label{lemma-leadingones-sampling-prob2}
For the LeadingOnes problem under bit-wise noise $(1,q)$ with $q\leq c_0\log n/n$ (where $c_0$ is a positive constant), if using sampling with $m\!=\!36n^{2c_0+4}$, it holds that
\begin{enumerate}[{(1)}]
  \item for any $x$ with $\mathrm{LO}(x)=i<n$ and $y$ with $\mathrm{LO}(y) \leq i-5c_0\log n$ or $i-5c_0\log n<\mathrm{LO}(y)\leq i-1 \wedge y_{i+1}=0$, $\mathrm{P}(\hat{f}(x) \leq \hat{f}(y))\leq 1/n^2$.
  \item for any $y$ with $\mathrm{LO}(y) <n$, $\mathrm{P}(\hat{f}(1^n) \leq \hat{f}(y))\leq 1/n^2$.
\end{enumerate}
\end{lemma}
\begin{myproof}
The proof is finished by deriving a lower bound on the expectation of $f^n(x)-f^n(y)$ and then applying Chebyshev's inequality. We first consider case~(1). For any $x$ with $\mathrm{LO}(x)=i<n$,
\begin{equation}
\begin{aligned}\label{LO,(1,q):E-uppercase1}
&\mathrm{E}(f^n(x))
\geq \sum_{j=1}^i (1-q)^{j-1}q\cdot (j-1) +(1-q)^{i}q\cdot (i+1)+(1-q)^{i+1}\cdot i,\\
&\mathrm{E}(f^n(x))\leq \sum_{j=1}^i (1-q)^{j-1}q\cdot (j-1)+(1-q)^{i}q\cdot n+(1-q)^{i+1}\cdot i.
\end{aligned}
\end{equation}
By applying these two inequalities, we get, for any $k<i<n$,
\begin{align} \label{LO,(1,q):Diff_E}
&\mathrm{E}(f^n(x)-f^n(y) \mid \mathrm{LO}(x)=i \wedge \mathrm{LO}(y)=k)\\
& \geq\frac{1}{q}((i-1)(1-q)^{i+1}-i(1-q)^i) +q(1-q)^i\cdot (i+1)+(1-q)^{i+1}\cdot i\\
& \quad -\frac{1}{q}((k-1)(1-q)^{k+1}-k(1-q)^k)-q(1-q)^k\cdot n -(1-q)^{k+1}\cdot k \nonumber\\
& =(1-q)^i\left(q+1-\frac{1}{q}\right)-(1-q)^k\left(1-\frac{1}{q}+qn-qk\right) \nonumber\\
& =(1-q)^k\left(\left(\frac{1}{q}-q-1\right)(1-(1-q)^{i-k})+q+qk-qn\right).
\end{align}
Thus, for any $x$ with $\mathrm{LO}(x)=i$ and $y$ with $\mathrm{LO}(y)=k\leq i-5c_0\log n$ (where $5c_0\log n\leq i<n$),
\begin{align}\label{mid-eq6}
\mathrm{E}(f^n(x)-f^n(y))
& \geq (1-q)^k\left(\left(\frac{1}{q}-2\right)(1-(1-q)^{5c_0\log n})-qn\right) \nonumber\\
& \geq (1-q)^k\left(\left(\frac{1}{q}-2\right)\left(1-\frac{1}{e^{5c_0q\log n}}\right)-qn\right) \nonumber\\
& \geq (1-q)^k\left(\left(\frac{1}{q}-2\right)\left(1-\frac{1}{1+5c_0q\log n}\right)-qn\right) \nonumber\\
& \geq (1-q)^k\left(\left(\frac{1}{q}-2\right)\frac{5c_0q\log n}{2}-qn\right) \nonumber\\
& \geq (1-q)^k\left(\frac{5}{2}c_0\log n-5\frac{(c_0\log n)^2}{n}-c_0\log n\right) \nonumber\\
& =(1-q)^k\left(\frac{3}{2}c_0\log n-o(1)\right) \geq \frac{c_0\log n}{en^{c_0}},
\end{align}
where the third inequality is by $e^x\geq 1+x$, the fourth is by $5c_0q \log n \leq 5(c_0\log n)^2/n \leq 1$ for sufficiently large $n$, the fifth is by $q \leq c_0\log n /n$, and the last is by
\begin{equation}\label{mid-eq4}
(1-q)^k\geq \left(1-\frac{c_0\log n}{n}\right)^{\!n-1}\!\geq \left(\frac{1}{e}\right)^{\!\frac{n-1}{n/(c_0\log n)-1}}\!\geq \left(\frac{1}{e}\right)^{\!c_0\log n+1} \!= \frac{1}{en^{c_0}}.
\end{equation}

When $i-5c_0\log n<\mathrm{LO}(y)\leq i-1 \wedge y_{i+1}=0$ (where $i\geq 1$), we calculate $\mathrm{E}(f^n(x)-f^n(y))$ by
$$
\mathrm{E}(f^n(x)-f^n(y))=\sum^{i}_{k=\mathrm{LO}(y)+1}\mathrm{E}(f^n(z^k)-f^n(z^{k-1})),
$$
where $z^i=x$, $z^k \; (\mathrm{LO}(y)<k<i)$ denotes one solution $z$ with $\mathrm{LO}(z)=k \wedge z_{i+1}=0$, and $z^{\mathrm{LO}(y)}=y$. We then give a lower bound on $\mathrm{E}(f^n(z^k)-f^n(z^{k-1}))$, where $\mathrm{LO}(y)+1\leq k\leq i$. For $\mathrm{E}(f^n(z^k))$, we directly use the lower bound in Eq.~(\refeq{LO,(1,q):E-uppercase1}) to get that
$$
\mathrm{E}(f^n(z^{k}))
\geq \sum_{j=1}^k (1-q)^{j-1}q\cdot (j-1) +(1-q)^{k}q\cdot (k+1)+(1-q)^{k+1}\cdot k.
$$
For $\mathrm{E}(f^n(z^{k-1}))$, instead of directly using the upper bound in Eq.~(\refeq{LO,(1,q):E-uppercase1}), we derive a tighter one:
\begin{equation}\label{LO,(1,q):E_upper}
\mathrm{E}(f^n(z^{k-1}))\leq \sum_{j=1}^{k-1} (1-q)^{j-1}q\cdot (j-1)+(1-q)^{k-1}q(q\cdot n+(1-q)\cdot i)+(1-q)^{k}\cdot (k-1).
\end{equation}
Note that the inequality is because when the $k-1$ leading 1-bits of $z^{k-1}$ keep unchanged and its $k$-th bit (which must be 0) is flipped, if the $(i+1)$-th bit (which is 0) is flipped, $f^n(z^{k-1})\leq n$; otherwise, $f^n(z^{k-1})\leq i$. By combining the above two inequalities, we get
\begin{align}
\mathrm{E}(f^n(z^{k})-f^n(z^{k-1}))
& \geq  (1-q)^{k-1}(1+q(k-i-1)-q^2(1+n-i)) \\
& =(1-q)^{k-1}(1-o(1)) \geq 1/(2en^{c_0}),
\end{align}
where the equality by $k \geq \mathrm{LO}(y)+1 >i-5c_0\log n+1$ and $q =O(\log n/n)$, and the last is by Eq.~(\refeq{mid-eq4}). Thus, we have
\begin{align}\label{mid-eq5}
\mathrm{E}(f^n(x)-f^n(y)) \geq (i-\mathrm{LO}(y))\cdot \frac{1}{2en^{c_0}} \geq \frac{1}{2en^{c_0}}.
\end{align}

For case~(1), by combining Eqs.~(\refeq{mid-eq6}) and~(\refeq{mid-eq5}), we get a unified lower bound $1/(2en^{c_0})$ on $\mathrm{E}(f^n(x)-f^n(y))$. As the analysis for $\mathrm{P}(\hat{f}(x) \leq \hat{f}(y))$ (i.e., Eq.~(\ref{LO,(p,1/n):Chebyshev})) in the proof of Lemma~\ref{lemma-leadingones-sampling-prob}, we can similarly use Chebyshev's inequality to derive that, noting that $m=36n^{2c_0+4}$ here,
\begin{align}\label{mid-eq7}
\mathrm{P}(\hat{f}(x) \leq \hat{f}(y))\leq \sigma^2/(m\mu^2) \leq 1/n^2,
\end{align}
where $\mu$ and $\sigma^2$ are the expectation and variance of $f^n(x)-f^n(y)$, respectively. Thus, case~(1) holds.

Then, we consider case~(2), that is, we are to analyze $\mathrm{P}(\hat{f}(1^n)\leq \hat{f}(y))$ with $\mathrm{LO}(y)<n$. We calculate $\mathrm{E}(f^n(1^n)-f^n(y))$ as follows:
$$
\mathrm{E}(f^n(1^n)-f^n(y))=\mathrm{E}(f^n(1^n)-f^n(1^{n-1}0))+\mathrm{E}(f^n(1^{n-1}0)-f^n(y)).
$$
It is easy to derive that
\begin{align}
&\mathrm{E}(f^n(1^n))=\sum_{j=1}^{n}(1-q)^{j-1}q\cdot(j-1)+(1-q)^n\cdot n,\\
&\mathrm{E}(f^n(1^{n-1}0))=\!\sum_{j=1}^{n-1}(1\!-\!q)^{j-1}q\cdot (j\!-\!1)+(1\!-\!q)^{n-1}q\cdot n+(1\!-\!q)^{n}\cdot (n\!-\!1).
\end{align}
Then, we have
\begin{align}
\mathrm{E}(f^n(1^n)-f^n(1^{n-1}0))=(1-q)^{n-1}(1-2q) \geq 1/(2en^{c_0}),
\end{align}
where the inequality is by Eq.~(\refeq{mid-eq4}) and $q=O(\log n/n)$. If $\mathrm{LO}(y)\leq n-1-5c_0\log n$, by Eq.~(\ref{mid-eq6}), we directly have $\mathrm{E}(f^n(1^{n-1}0)-f^n(y))\geq \frac{c_0\log n}{en^{c_0}} \geq 0$. If $\mathrm{LO}(y)\geq n-5c_0\log n$, $\mathrm{E}(f^n(1^{n-1}0)-f^n(y))$ is calculated as follows:
$$
\mathrm{E}(f^n(1^{n-1}0)-f^n(y)) =\sum^{n-1}_{k=\mathrm{LO}(y)+1} \mathrm{E}(f^n(z^k)-f^n(z^{k-1})),
$$
where $z^{n-1}=1^{n-1}0$, $z^k \; (\mathrm{LO}(y)<k<n-1)$ denotes one solution $z$ with $\mathrm{LO}(z)=k$, and $z^{\mathrm{LO(y)}}=y$. By Eq.~(\refeq{LO,(1,q):Diff_E}), we have, for any $\mathrm{LO}(y)+1 \leq k<n$,
\begin{align}
&\mathrm{E}(f^n(z^{k})-f^n(z^{k-1}))\geq (1-q)^{k-1}(1-q^2+q(k-1-n)) \\
& \geq (1-q)^{k-1}(1-q^2-5c_0q\log n)=(1-q)^{k-1}(1-o(1)) \geq 0,
\end{align}
which implies that $\mathrm{E}(f^n(1^{n-1}0)-f^n(y))\geq 0$. Then, we get
$$
\mathrm{E}(f^n(1^n)-f^n(y))=\mathrm{E}(f^n(1^n)-f^n(1^{n-1}0))+\mathrm{E}(f^n(1^{n-1}0)-f^n(y))
\ge \frac{1}{2en^{c_0}}.
$$
As the analysis in case~(1) (i.e., Eq.~(\refeq{mid-eq7})), we can get that $$\mathrm{P}(\hat{f}(1^n) \leq \hat{f}(y)) \leq 1/n^2.$$
Thus, case~(2) holds.\vspace{0.8em}
\end{myproof}

The following theorem shows that by using sampling, the (1+1)-EA can solve LeadingOnes under bit-wise noise $(1,q)$ in polynomial time when $q$ is in the range of $O(\log n/n)$.

\begin{theorem}\label{theo-leadingones-sampling-upper2}
 For the (1+1)-EA on LeadingOnes under bit-wise noise $(1,q)$ with $q=O(\log n/n)$, if using sampling, there exists some $m= O(poly(n))$ such that the expected running time is polynomial.
\end{theorem}
\begin{myproof}
Since $q=O(\log n/n)$, there exists a positive constant $c_0$ such that for all $n$ large enough, $q\leq c_0\log n/n$. We prove that if using sampling with $m=36n^{2c_0+4}$, the expected running time is polynomial.

The proof is similar to that of Theorem~\ref{theo-leadingones-sampling-upper1}. The distance function $V(x)$ is defined as, for any $x$ with $\mathrm{LO}(x)=i$, $V(x)=\left(1+\frac{c}{n}\right)^n-\left(1+\frac{c}{n}\right)^i$, where $c=30c_0\log n+7$. Assume that currently $\mathrm{LO}(x)=i$, where $0\leq i \leq n-1$. For the positive drift $\mathrm{E}^+$, we consider that only the $(i+1)$-th bit (i.e., the first 0-bit) of $x$ is flipped in mutation. If $\mathrm{LO}(x')=n$, $\mathrm{P}(\hat{f}(x') \leq \hat{f}(x)) \leq 1/n^2$ by case~(2) of Lemma~\ref{lemma-leadingones-sampling-prob2}. If $\mathrm{LO}(x')<n$, it must hold that $\mathrm{LO}(x)\leq \mathrm{LO}(x')-1 \wedge x_{\mathrm{LO}(x')+1}=x'_{\mathrm{LO}(x')+1}=0$, since $x$ and $x'$ are the same except the $(\mathrm{LO}(x)+1)$-th bit. By case~(1) of Lemma~\ref{lemma-leadingones-sampling-prob2}, $\mathrm{P}(\hat{f}(x') \leq \hat{f}(x)) \leq 1/n^2$. Thus, the probability of accepting the offspring solution $x'$ is $\mathrm{P}(\hat{f}(x') \geq \hat{f}(x)) \geq 1-1/n^2 \geq 1/2$. The positive drift then can be lower bounded by
$$
\mathrm{E}^+ \geq  \left(1-\frac{1}{n}\right)^{n-1}\frac{1}{n}\cdot \frac{1}{2} \cdot \frac{c}{n}\left(1+\frac{c}{n}\right)^i\geq \frac{c}{6n^2}\left(1+\frac{c}{n}\right)^i.
$$

For the negative drift $\mathrm{E}^{-}$, we need to consider $\mathrm{LO}(x')<\mathrm{LO}(x)=i$. We further divide $\mathrm{LO}(x')<i$ into two cases. If $\mathrm{LO}(x')\leq i-5c_0\log n$ or $i-5c_0\log n<\mathrm{LO}(x')\leq i-1 \wedge x'_{i+1}=0$, then $\mathrm{P}(\hat{f}(x') \geq \hat{f}(x)) \leq 1/n^2$ by case~(1) of Lemma~\ref{lemma-leadingones-sampling-prob2}. If $i-5c_0\log n<\mathrm{LO}(x')\leq i-1 \wedge x'_{i+1}=1$, we consider the probability of generating $x'$ by mutation on $x$. Since it is necessary to flip the $(i+1)$-th bit of $x$ and at least one 1-bit in positions $i-5c_0\log n+2$ to $i$ simultaneously, $$\sum_{x': i-5c_0\log n<\mathrm{LO}(x')\leq i-1 \wedge x'_{i+1}=1}\mathrm{P}_{mut}(x,x') \leq \frac{5c_0\log n}{n^2}.$$
Then, we get
\begin{align}\label{LO,(p,1/n):E-}
& \mathrm{E}^-\leq \left(\left(1+\frac{c}{n}\right)^i-1\right)\cdot \left(1\cdot \frac{1}{n^2}
 +\frac{5c_0\log n}{n^2}\cdot 1\right) \leq \frac{5c_0\log n+1}{n^2}\left(1+\frac{c}{n}\right)^i.
\end{align}

By subtracting $\mathrm{E}^-$ from $\mathrm{E}^+$, we have, noting that $c=30c_0\log n+7$,
\begin{align}\label{LO,(p,1/n):E}
\mathrm{E}(V(\xi_t)-V(\xi_{t+1}) \mid \xi_t=x)
&\geq \left(1+\frac{c}{n}\right)^i \cdot \left(\frac{c}{6n^2}-\frac{5c_0\log n+1}{n^2}\right) \geq \frac{1}{6n^2}.
\end{align}
Note that $V(x) \le (1+\frac{c}{n})^n \le e^c = e^{30c_0\log n+7}=
n^{O(1)}$. By Theorem~\ref{additive-drift}, we have $\mathrm{E}(\tau\mid \xi_0)\le 6n^2\cdot n^{O(1)}$. Since each iteration of the (1+1)-EA using sampling takes $2m=72n^{2c_0+4}$ number of fitness evaluations, the expected running time is polynomial.\vspace{0.8em}
\end{myproof}

For bit-wise noise $(1,q)$ with $q=\omega (\log n/n)$, we apply the negative drift theorem (i.e., Theorem~\ref{simplified-drift}) to prove that using sampling still cannot guarantee a polynomial running time.

\begin{theorem} \label{theo-leadingones-sampling-lower1}
For the (1+1)-EA on LeadingOnes under bit-wise noise $(1,q)$ with $q=\omega (\log n/n)$, if using sampling with any $m= O(poly(n))$, the expected running time is exponential.
\end{theorem}
\begin{myproof}
We use Theorem~\ref{simplified-drift} to prove it. Let $X_t=|x|_0$ be the number of 0-bits of the solution $x$ after $t$ iterations of the algorithm. We consider the interval $[0,n/50]$, that is, the parameters $a=0$ and $b=n/50$ in Theorem~\ref{simplified-drift}. Then, we analyze the drift $\mathrm{E}(X_t-X_{t+1}\mid X_t=i)$ for $1\leq i<n/50$. As in the proof of Theorem~\ref{leadingones-noise-2-superpoly}, we divide the drift into two parts: positive $\mathrm{E}^+$ and negative $\mathrm{E}^-$. That is,
$$
\mathrm{E}(X_t-X_{t+1}\mid X_t=i)=\mathrm{E}^+ - \mathrm{E}^-.
$$
For the positive drift, we can use the same analysis as that (i,e., Eq.~(\refeq{eq-positive-drift-1})) in the proof of Theorem~\ref{leadingones-noise-2-superpoly} to derive that $\mathrm{E}^+\leq i/n<1/50$. This is because the offspring solution $x'$ is optimistically assumed to be always accepted in the analysis of Eq.~(\refeq{eq-positive-drift-1}), and thus the change of noise and the use of sampling will not affect the analysis.

For the negative drift, we need to consider that the number of 0-bits is increased. To derive a lower bound on $\mathrm{E}^-$, we only consider the $n-i$ cases where only one 1-bit of $x$ is flipped, which happens with probability $\frac{1}{n}(1-\frac{1}{n})^{n-1}$. Let $x^j$ denote the solution that is generated by flipping only the $j$-th bit of $x$. Then, we have
\begin{align}
\mathrm{E}^-\geq \sum_{j: x_j=1}\frac{1}{n}\left(1-\frac{1}{n}\right)^{n-1}\cdot \mathrm{P}(\hat{f}(x^j)\geq \hat{f}(x)) \cdot (i+1-i).
\end{align}
We then investigate $\mathrm{P}(\hat{f}(x^j)\ge \hat{f}(x))$. Let $\mathrm{F}(y)$ denote the event that when evaluating the noisy fitness of a solution $y$, at least one 1-bit in its first $n/3$ positions is flipped by noise. Note that there must exist 1-bits in the first $n/3$ positions of $x$, since $|x|_0=i<n/50$. For any $x^j$ with $j>n/3$, its first $n/3$ bits are the same as that of $x$. If both the events $\mathrm{F}(x)$ and $\mathrm{F}(x^j)$ happen, $f^n(x)<n/3 \wedge f^n(x^j)<n/3$, and the last $(2n)/3$ bits of $x$ and $x^j$ will not affect their noisy fitness. Thus, for any $x^j$ with $j>n/3$, $f^n(x^j)$ has the same distribution as $f^n(x)$ conditioned on $\mathrm{F}(x)\cap \mathrm{F}(x^j)$. When estimating $\hat{f}(y)$ of a solution $y$ by sampling, let $\mathrm{F}_t(y)$ denote the event $\mathrm{F}(y)$ in the $t$-th independent noisy evaluation of $y$. Thus, for all $j>n/3$, $\sum_{t=1}^mf^n_t(x)$ and $\sum_{t=1}^mf^n_t(x^j)$ have the same distribution conditioned on $(\cap_{t=1}^m\mathrm{F}_t(x))\cap(\cap_{t=1}^m\mathrm{F}_t(x^j))$. Since $\hat{f}(x)=\frac{1}{m}\sum_{t=1}^mf^n_t(x)$ and $\hat{f}(x^j)=\frac{1}{m}\sum_{t=1}^mf^n_t(x^j)$, we have
$$\mathrm{P}(\hat{f}(x^j)\geq \hat{f}(x) \mid (\cap_{t=1}^m\mathrm{F}_t(x))\cap(\cap_{t=1}^m\mathrm{F}_t(x^j)))\geq 1/2.$$
Since $|x|_0=i<n/50$, there are at least $n/3-n/50 \geq n/4$ number of 1-bits in the first $n/3$ positions of $x$, which implies that the probability $\mathrm{P}(\mathrm{F}_t(x))$ of the event $\mathrm{F}_t(x)$ happening is at least $1-(1-q)^{n/4}$. Furthermore, $\mathrm{P}(\mathrm{F}_t(x^j))=\mathrm{P}(\mathrm{F}_t(x))$ for $j>n/3$, since $x$ and $x^j$ have the same first $n/3$ bits. Thus,
\begin{align}
&\mathrm{P}((\cap_{t=1}^m\mathrm{F}_t(x))\cap(\cap_{t=1}^m\mathrm{F}_t(x^j)))\geq (1-(1-q)^{n/4})^{2m} \\
&\ge \left(1-\frac{1}{e^{nq/4}}\right)^{2m} \ge \left(\frac{1}{e}\right)^{2m/(e^{nq/4}-1)} = \left(\frac{1}{e}\right)^{2m/n^{\omega(1)}} \ge \frac{1}{e},
\end{align}
where the equality is by $q=\omega(\log n/n)$ and the last inequality is by $m=O(poly(n))$. By the law of total probability, we have, for all $j>n/3$,
\begin{align*}
\mathrm{P}(\hat{f}(x^j)\ge \hat{f}(x))&\ge \mathrm{P}(\hat{f}(x^j)\ge \hat{f}(x) \mid (\cap_{t=1}^m\mathrm{F}_t(x))\cap(\cap_{t=1}^m\mathrm{F}_t(x^j)))\\
&\quad \cdot \mathrm{P}((\cap_{t=1}^m\mathrm{F}_t(x))\cap(\cap_{t=1}^m\mathrm{F}_t(x^j)))\ge \frac{1}{2e}.
\end{align*}
Then, we can get a lower bound on the negative drift:
\begin{align*}
\mathrm{E}^-
&\ge \sum_{j: j>n/3 \wedge x_j=1}\frac{1}{n}\left(1-\frac{1}{n}\right)^{n-1}\cdot P(\hat{f}(x^j)\ge \hat{f}(x)) \cdot 1 \\
&\ge \left(\frac{2n}{3}-\frac{n}{50}\right)\cdot \frac{1}{en}\cdot \frac{1}{2e} \ge \frac{1}{36}.
\end{align*}

By subtracting $\mathrm{E}^-$ from $\mathrm{E}^+$, we have, for $1 \leq i< n/50$,
$$
\mathrm{E}(X_t-X_{t+1}\mid X_t=i)=\mathrm{E}^+-\mathrm{E}^-\leq \frac{1}{50}-\frac{1}{36}.
$$
Thus, condition~(1) of Theorem~\ref{simplified-drift} holds. It is easy to verify that condition (2) of Theorem~\ref{simplified-drift} holds with $\delta =1$ and $r(l)=2$. Note that $l=b-a=n/50$. By Theorem~\ref{simplified-drift}, we get that the expected number of iterations is exponential, and then the expected running time is also exponential.\vspace{0.8em}
\end{myproof}

For the one-bit noise model, we prove in Theorem~\ref{theo-leadingones-sampling-upper3} that the (1+1)-EA using sampling can always solve the LeadingOnes problem in polynomial time. The proof is finished by applying the additive drift theorem. Lemma~\ref{lemma-leadingones-sampling-prob3} gives some bounds on the probability $\mathrm{P}(\hat{f}(x')-\hat{f}(x))$ of accepting the offspring solution $x'$, which will be used in the proof of Theorem~\ref{theo-leadingones-sampling-upper3}.

From case~(2) of Lemma~\ref{lemma-leadingones-sampling-prob3}, we can observe that when the solution is close to the optimum $1^n$, the probability of accepting $1^n$ is small, which is different from the situation in bit-wise noise (as shown in case~(2) of Lemmas~\ref{lemma-leadingones-sampling-prob} and~\ref{lemma-leadingones-sampling-prob2}). If directly using the distance function constructed in the proof of Theorems~\ref{theo-leadingones-sampling-upper1} and~\ref{theo-leadingones-sampling-upper2}, this small acceptance probability will make the positive drift $\mathrm{E}^+$ not large enough, and then the condition of the additive drift theorem is unsatisfied. To address this issue, our idea is to re-design the distance function such that the distance from non-optimal solutions to the optimum $1^n$ is much larger than that between non-optimal solutions. Then, the small probability of accepting $1^n$ can be compensated by the significant decrease on the distance after accepting $1^n$; thus the positive drift can still be large enough to make the condition of the additive drift theorem hold.

Note that in the proof of Lemma~\ref{lemma-leadingones-sampling-prob3}, we use Berry-Esseen inequality~\cite{shevtsova2007sharpening} and Hoeffding's inequality, instead of Chebyshev's inequality used in the proof of Lemmas~\ref{lemma-leadingones-sampling-prob} and~\ref{lemma-leadingones-sampling-prob2}. When the solution $x$ is close to the optimum $1^n$, the expectation of $f^n(1^n)-f^n(x)$ is lower bounded by a negative value. Thus, for deriving a lower bound on the probability $\mathrm{P}(\hat{f}(1^n) \geq \hat{f}(x))$, Chebyshev's inequality fails, while we apply Berry-Esseen inequality~\cite{shevtsova2007sharpening}. The analysis shows that a moderate sample size $m=4n^4\log n/15$ can make this probability not too small. With this sample size, to derive a small enough upper bound on the probability $\mathrm{P}(\hat{f}(x)\leq \hat{f}(y))$ for two solutions $x$ and $y$ with $\mathrm{E}(f^n(x)-f^n(y))>0$, we have to use Hoeffding's inequality, which is tighter than Chebyshev's inequality.

\begin{lemma}\label{lemma-leadingones-sampling-prob3}
For the LeadingOnes problem under one-bit noise, if using sampling with $m=4n^4\log n/15$, it holds that
\begin{enumerate}[{(1)}]
  \item for any $x$ with $\mathrm{LO}(x)=i<n$ and $y$ with $\mathrm{LO}(y) \leq i-2$ or $\mathrm{LO}(y)=i-1 \wedge y_{i+1}=0$, $\mathrm{P}(\hat{f}(x) \leq \hat{f}(y))\leq 2n^{-\frac{2}{15}}$; furthermore, if $i \leq n-4$, $\mathrm{P}(\hat{f}(x) \leq \hat{f}(y))\leq 2n^{-\frac{32}{15}}$.
  \item for any $y$, if $\mathrm{LO}(y) \leq n-4$, $\mathrm{P}(\hat{f}(1^n) \leq \hat{f}(y))\leq 2n^{-\frac{8}{15}}$; if $\mathrm{LO}(y) \in \{n-3,n-2,n-1\}$, $\mathrm{P}(\hat{f}(1^n) \geq \hat{f}(y))\geq 1/n^2$.
\end{enumerate}
\end{lemma}
\begin{myproof}
The proof is finished by deriving a lower bound on the expectation of $f^n(x)-f^n(y)$ and then applying Hoeffding's inequality or Berry-Esseen inequality~\cite{shevtsova2007sharpening}. We first consider case~(1). For any $x$ with $\mathrm{LO}(x)=i<n$,
\begin{align}\label{LO,Onebit:E_lower}
&\mathrm{E}(f^n(x))\ge (1-p)\cdot i+\sum_{j=1}^i\frac{p}{n}\cdot(j-1)+\frac{p}{n}\cdot (i+1)+p\frac{n-i-1}{n}\cdot i,\\
&\mathrm{E}(f^n(x))\le (1-p)\cdot i+\sum_{j=1}^{i}\frac{p}{n}\cdot(j-1)+\frac{p}{n}\cdot n+p\frac{n-i-1}{n}\cdot i.
\end{align}
By applying these two inequalities, we have, for any $k< i<n$,
\begin{align}\label{mid-eq11}
&\mathrm{E}(f^n(x)-f^n(y)\mid \mathrm{LO}(x)=i \wedge \mathrm{LO}(y)=k)\\
&\ge (1-p)(i-k)+\frac{p}{n}\left((i-k)\left(n-\frac{i+k+3}{2}\right)+i+1-n\right).
\end{align}
Thus, for any $x$ with $\mathrm{LO}(x)=i$ and $y$ with $\mathrm{LO}(y)=k \leq i-2$ (where $2 \leq i<n$), we have
\begin{align}
\mathrm{E}(f^n(x)-f^n(y))\geq 2(1-p)+\frac{p}{n}(n-i) \geq \frac{n-i}{n}.
\end{align}
When $\mathrm{LO}(y)=i-1 \wedge y_{i+1}=0$, if we directly use Eq.~(\refeq{mid-eq11}), we will get a lower bound $1-p$, which is 0 for $p=1$. Since $y_{i+1}=0$, we actually can get an exact value of $\mathrm{E}(f^n(y))$:
\begin{align}
\mathrm{E}(f^n(y))= (1-p)\cdot (i-1)+\sum_{j=1}^{i-1}\frac{p}{n}\cdot(j-1)+\frac{p}{n}\cdot i+p\frac{n-i}{n}\cdot (i-1).
\end{align}
By combining this equality and the lower bound in Eq.~(\refeq{LO,Onebit:E_lower}), we have that, for any $x$ with $\mathrm{LO}(x)=i$ and $y$ with $\mathrm{LO}(y) = i-1 \wedge y_{i+1}=0$ (where $1 \leq i<n$),
\begin{align} \label{LO,Onebit:Diff_E-A}
\mathrm{E}(f^n(x)-f^n(y))
&\ge 1-p+\frac{p}{n}(i-1)+\frac{p}{n}+\frac{p}{n}(n-2i) \geq \frac{n-i}{n}.
\end{align}
Thus, we have a unified lower bound $(n-i)/n$ on $\mathrm{E}(f^n(x)-f^n(y))$ for case~(1). Denote $\mathrm{E}(f^n(x)-f^n(y))$ as $\mu$. We have $\mu \geq (n-i)/n\geq 1/n$. Since $\hat{f}(x)$ is the average of $m=4n^4\log n/15$ independent evaluations, it is easy to verify that $\mathrm{E}(\hat{f}(x)-\hat{f}(y))=\mu$. Furthermore, $|f^n(x)-f^n(y)| \leq n$. By Hoeffding's inequality, we get
\begin{equation} \label{LO,Onebit:Hoeffding}
\mathrm{P}(\hat{f}(x)\leq \hat{f}(y)) \le \mathrm{P}(|\hat{f}(x)-\hat{f}(y) -\mu|\ge \mu)
\le 2e^{-\frac{2m^2\mu^2}{m(2n)^2}}
\le 2n^{-\frac{2}{15}}.
\end{equation}
When $i \leq n-4$, $\mu \geq (n-i)/n\geq 4/n$. By applying this lower bound to the above inequality, we get $$\mathrm{P}(\hat{f}(x)\leq \hat{f}(y)) \leq 2n^{-\frac{32}{15}}.$$ Thus, case~(1) holds.

For case~(2), we are to analyze $\mathrm{P}(\hat{f}(1^n) \leq \hat{f}(y))$ or $\mathrm{P}(\hat{f}(1^n) \geq \hat{f}(y))$, where $\mathrm{LO}(y)<n$. $\mathrm{E}(f^n(1^n))$ can be calculated as follows:
\begin{align} \label{LO,Onebit:E(n)}
\mathrm{E}(f^n(1^n))= (1-p)\cdot n+\sum_{j=1}^n\frac{p}{n}\cdot(j-1).
\end{align}
By combining this equality and the upper bound in Eq.~(\ref{LO,Onebit:E_lower}), we get, for any $y$ with $\mathrm{LO}(y)=k<n$,
\begin{align} \label{LO,Onebit:E(n-k)}
\mathrm{E}(f^n(1^n)-f^n(y)) \geq (n-k)\left(1-p+\frac{p}{n}\frac{n-k-3}{2}\right).
\end{align}
When $k \leq n-4$, $\mathrm{E}(f^n(1^n)-f^n(y))\geq 2/n$. By Hoeffding's inequality, we get $$\mathrm{P}(\hat{f}(1^n) \leq \hat{f}(y)) \leq 2n^{-\frac{8}{15}}.$$ When $k \in \{n-3,n-2,n-1\}$,
$$
\mu=\mathrm{E}(f^n(1^n)-f^n(y))\geq 1-p-p/n.
$$
If $p\leq 1-2/n$, we have $\mu\geq 1/n$. By Hoeffding's inequality,
\begin{align}\label{mid-eq10}
\mathrm{P}(\hat{f}(1^n) \geq \hat{f}(y)) \geq 1-2n^{-\frac{2}{15}}.
\end{align}
If $p> 1-2/n$, $\mu \geq -1/n$. We then use Berry-Esseen inequality~\cite{shevtsova2007sharpening} to derive a lower bound on $\mathrm{P}(\hat{f}(1^n) \geq \hat{f}(y))$. Let $Y=f^n(1^n)-f^n(y)-\mu$. Note that $\mathrm{E}(Y)=0$. Denote the variance $\mathrm{Var}(Y)$ as $\sigma^2$. Then, $\sigma^2=\mathrm{Var}(f^n(1^n))+\mathrm{Var}(f^n(y))$. For $\mathrm{Var}(f^n(1^n))$, it can be calculated as follows:
\begin{align*}
\mathrm{Var}(f^n(1^n))
& =\mathrm{E}((f^n(1^n))^2)-(\mathrm{E}(f^n(1^n))^2 \\
& =(1-p)n^2+\frac{p}{n}\frac{(n-1)n(2n-1)}{6}-\left((1-p)n+\frac{p}{n}\frac{n(n-1)}{2}\right)^2 \\
& \ge \frac{p(n-1)(n-1/2)}{3}-\left(1+\frac{p(n-1)}{2}\right)^2 \ge \frac{n^2}{14},
\end{align*}
where the last inequality holds because for $p>1-2/n$ and $n$ being large enough, the minimum is reached when $p=1$. Thus, we get $\sigma^2\geq n^2/14$. Since $-2n\le Y\le 2n$, $\rho=\mathrm{E}(|Y|^3)\le 8n^3$. Note that $\hat{f}(1^n)-\hat{f}(y)-\mu$ is the average of $m$ independent random variables, which have the same distribution as $Y$. By Berry-Esseen inequality~\cite{shevtsova2007sharpening}, we have
\begin{align}\label{mid-eq8}
\mathrm{P}\left(\frac{(\hat{f}(1^n)-\hat{f}(y)-\mu)\sqrt{m}}{\sigma}\le x\right)-\Phi(x)\le \frac{\rho}{\sigma^3\sqrt{m}},
\end{align}
where $\Phi(x)$ is the cumulative distribution function of the standard normal distribution. Thus, for $p>1-2/n$,
\begin{align}\label{mid-eq9}
\mathrm{P}(\hat{f}(1^n) \geq \hat{f}(y))&=P(\hat{f}(1^n)- \hat{f}(y)-\mu\ge -\mu)\\
&\geq 1-P(\hat{f}(1^n)- \hat{f}(y)-\mu\leq -\mu)\\
& \ge 1-\Phi\left(\frac{-\mu\sqrt{m}}{\sigma}\right)-\frac{\rho}{\sigma^3\sqrt{m}} \\
& =\Phi\left(\frac{\mu\sqrt{m}}{\sigma}\right)-\frac{\rho}{\sigma^3\sqrt{m}} \ge\Phi\left(-\frac{\sqrt{m}}{n\sigma}\right)-\frac{\rho}{\sigma^3\sqrt{m}} \\
& \ge\Phi\left(-\sqrt{\frac{56}{15}\log n}\right)-O\left(\frac{1}{n^2\sqrt{\log n}}\right) \\
& \ge \frac{1}{\sqrt{2\pi}}\int_{-2\sqrt{\log n}}^{-\sqrt{\frac{56}{15}\log n}}e^{-\frac{t^2}{2}}\mathrm{d}t -O\left(\frac{1}{n^2\sqrt{\log n}}\right) \ge\frac{1}{n^2},
\end{align}
where the second inequality is by Eq.~(\refeq{mid-eq8}), the third inequality is $\mu \geq -1/n$, the fourth is by $m=4n^4\log n/15$, $\sigma^2 \geq n^2/14$ and $\rho \leq 8n^3$, and the last holds with sufficiently large $n$. According to Eqs.~(\refeq{mid-eq10}) and~(\refeq{mid-eq9}), we get that, when $\mathrm{LO}(y) \in \{n-3,n-2,n-1\}$,
$$
\mathrm{P}(\hat{f}(1^n) \geq \hat{f}(y)) \geq 1/n^2.
$$
Thus, case~(2) holds.\vspace{0.8em}
\end{myproof}

The following theorem shows that for one-bit noise, using sampling with $m=4n^4\log n/15$ is sufficient to make the (1+1)-EA solve LeadingOnes in polynomial time.

\begin{theorem}\label{theo-leadingones-sampling-upper3}
For the (1+1)-EA on LeadingOnes under one-bit noise, if using sampling with $m=4n^4\log n/15$, the expected running time is polynomial.
\end{theorem}
\begin{myproof}
We use Theorem~\ref{additive-drift} to prove it. We first construct a distance function $V(x)$ as, for any $x$ with $\mathrm{LO}(x)=i$,
 \[
 V(x)=\left\{\begin{array}{ll}
 n-i/n^6, & \quad i\leq n-1; \\
 0,  & \quad i=n.
\end{array}\right.
 \]
Then, we investigate $\mathrm{E}(V(\xi_t)-V(\xi_{t+1}) \mid \xi_t=x)$ for any $x$ with $\mathrm{LO}(x)=i<n$. We divide the drift into two parts: positive $\mathrm{E}^+$ and negative $\mathrm{E}^-$. That is, $$\mathrm{E}(V(\xi_t)-V(\xi_{t+1}) \mid \xi_t=x)=\mathrm{E}^+-\mathrm{E}^-.$$

For $i\leq n-4$, the lower bound analysis on $\mathrm{E}^+$ is similar to that in the proof of Theorem~\ref{theo-leadingones-sampling-upper1}. We only consider that the $(i+1)$-th bit of $x$ is flipped and the other bits keep unchanged, whose probability is $\frac{1}{n}(1-\frac{1}{n})^{n-1}$. The offspring solution $x'$ is the same as $x$ except the $(i+1)$-th bit, and $\mathrm{LO}(x')>\mathrm{LO}(x)=i$. According to definition of $V(x)$, we know that the decrease on the distance is at least $1/n^6$. If $\mathrm{LO}(x')=n$, $\mathrm{P}(\hat{f}(x') \leq \hat{f}(x)) \leq 2n^{-\frac{8}{15}}$ by case~(2) of Lemma~\ref{lemma-leadingones-sampling-prob3}. If $\mathrm{LO}(x')<n$, it must hold that $\mathrm{LO}(x)\leq \mathrm{LO}(x')-1 \wedge x_{\mathrm{LO}(x')+1}=x'_{\mathrm{LO}(x')+1}=0$. By case~(1) of Lemma~\ref{lemma-leadingones-sampling-prob3}, $\mathrm{P}(\hat{f}(x') \leq \hat{f}(x)) \leq 2n^{-\frac{2}{15}}$. Thus, the probability $\mathrm{P}(\hat{f}(x') \geq \hat{f}(x))$ of accepting the offspring solution $x'$ is at least $1-\max\{2n^{-\frac{8}{15}},2n^{-\frac{2}{15}}\} \geq 1/2$, where the inequality holds with sufficiently large $n$. Then, $\mathrm{E}^+$ can be lower bounded as follows:
\begin{align} \label{LO,(p,1/n):E+}
\mathrm{E}^+ \geq  \left(1-\frac{1}{n}\right)^{n-1}\frac{1}{n}\cdot \frac{1}{2} \cdot \frac{1}{n^6}\geq \frac{1}{6n^7}.
\end{align}
For the negative drift $\mathrm{E}^{-}$, we need to consider $\mathrm{LO}(x')<\mathrm{LO}(x)=i$. We further divide $\mathrm{LO}(x')<i$ into two cases. If $\mathrm{LO}(x')\leq i-2$ or $\mathrm{LO}(x')= i-1 \wedge x'_{i+1}=0$, then $\mathrm{P}(\hat{f}(x') \geq \hat{f}(x)) \leq 2n^{-\frac{32}{15}}$ by case~(1) of Lemma~\ref{lemma-leadingones-sampling-prob3} (note that $i\leq n-4$ here), and $V(x')-V(x) \leq V(0^n)-V(x)= i/n^6$. If $\mathrm{LO}(x')= i-1 \wedge x'_{i+1}=1$, then $\sum_{x': \mathrm{LO}(x')= i-1 \wedge x'_{i+1}=1}\mathrm{P}_{mut}(x,x') \leq 1/n^2$ since it is necessary to flip the $i$-th and the $(i+1)$-th bits of $x$ in mutation, and $V(x')-V(x)=1/n^6$. Then, $\mathrm{E}^{-}$ can be upper bounded by as follows:
\begin{align*}
\mathrm{E}^- \le\frac{i}{n^6}\cdot 1\cdot \frac{2}{n^{\frac{32}{15}}}+\frac{1}{n^6}\cdot\frac{1}{n^2}\cdot 1=o\left(\frac{1}{n^7}\right).
\end{align*}
By subtracting $\mathrm{E}^-$ from $\mathrm{E}^+$, we have
\begin{align}\label{LO,(p,1/n):E}
\mathrm{E}(V(\xi_t)-V(\xi_{t+1}) \mid \xi_t=x)=\mathrm{E}^+-\mathrm{E}^-\geq \frac{1}{12n^7}.
\end{align}

For $i\in \{n-3,n-2,n-1\}$, we use a trivial upper bound $i/n^6$ on $\mathrm{E}^-$. For the positive drift, we consider that the offspring solution $x'$ is the optimal solution $1^n$, whose probability is at least $\frac{1}{n^3}(1-\frac{1}{n})^{n-3} \geq \frac{1}{en^3}$ since at most three bits of $x$ need to be flipped. The probability of accepting $1^n$ is at least $\frac{1}{n^2}$ by case~(2) of Lemma~\ref{lemma-leadingones-sampling-prob3}. The distance decrease is at least $n-\frac{n-1}{n^6} \geq \frac{n}{2}$. Thus, $\mathrm{E}^+ \geq \frac{1}{en^3} \cdot \frac{1}{n^2} \cdot \frac{n}{2} \geq \frac{1}{2en^4}$. By subtracting $\mathrm{E}^-$ from $\mathrm{E}^+$, we have
$$
\mathrm{E}(V(\xi_t)-V(\xi_{t+1}) \mid \xi_t=x)\geq \frac{1}{2en^4}-\frac{i}{n^6}= \frac{1}{2en^4}-O\left(\frac{1}{n^5}\right) \geq \frac{1}{4en^4}.
$$

By combing the above analyses for $i\leq n-4$ and $i\in \{n-3,n-2,n-1\}$, we get a unified lower bound $1/(12n^7)$ on $\mathrm{E}(V(\xi_t)-V(\xi_{t+1}) \mid \xi_t=x)$. Since $V(x)\leq n$, we have $\mathrm{E}(\tau\mid \xi_0)=O(n^{8})$ by Theorem~\ref{additive-drift}. Each iteration of the (1+1)-EA using sampling takes $2m=8n^4\log n/15$ number of fitness evaluations, thus the expected running time is polynomial.\vspace{0.8em}
\end{myproof}

Therefore, we have shown that the (1+1)-EA using sampling can always solve LeadingOnes in polynomial time under bit-wise noise $(p,\frac{1}{n})$ (i.e., Theorem~\ref{theo-leadingones-sampling-upper1}) or one-bit noise (i.e., Theorem~\ref{theo-leadingones-sampling-upper3}); while under bit-wise noise $(1,q)$, the tight range of $q$ allowing a polynomial time is $O(\log n/n)$ (i.e., Theorems~\ref{theo-leadingones-sampling-upper2} and~\ref{theo-leadingones-sampling-lower1}). The reason why sampling is ineffective under bit-wise noise $(1,q)$ with $q=\omega(\log n/n)$ is similar to that observed in the analysis of OneMax under bit-wise noise $(1,q)$ with $q=1/2-1/n^{\omega(1)}$ or $q \geq 1/2$. For two solutions $x$ and $y$ with $f(x) > f(y)$, we can find from the calculation of $\mathrm{E}(f^n(x)-f^n(y))$ in the proof of Lemma~\ref{lemma-leadingones-sampling-prob2} that when $q=\omega(\log n/n)$, $\mathrm{E}(f^n(x)-f^n(y))$ will be very small (since $(1-q)^{n-1}=1/n^{\omega(1)}$) or even negative; thus a polynomial sample size is not sufficient to make the probability of accepting the true worse solution $y$ small enough, or it will increase the probability. For the situation where sampling is effective, the analysis on LeadingOnes is a little different from that on OneMax. On OneMax, $\mathrm{E}(f^n(x)-f^n(y))$ is always sufficiently large when $f(x)>f(y)$; thus sampling can make the probability of accepting the true worse solution $y$ small enough and then work. While on LeadingOnes, $\mathrm{E}(f^n(x)-f^n(y))$ is sufficiently large in most cases instead of all cases when $f(x)>f(y)$, but a few exceptions do not affect the effectiveness of sampling.

\section{Conclusion}

In this paper, we theoretically study the (1+1)-EA solving the OneMax and LeadingOnes problems under bit-wise noise, which is characterized by a pair $(p,q)$ of parameters. We derive the ranges of $p$ and $q$ for the running time being polynomial and super-polynomial, respectively. The previously known parameter ranges for the (1+1)-EA solving LeadingOnes under one-bit noise are also improved. Considering that the (1+1)-EA is efficient only under low noise levels, we further analyze the robustness of sampling to noise. We prove that for both bit-wise noise and one-bit noise, using sampling can significantly enlarge the range of noise parameters allowing a polynomial running time. In the future, we shall improve the currently derived bounds on LeadingOnes, as they do not cover the whole range of noise parameters. For proving polynomial upper bounds on the expected running time by using sampling, we only give a sufficiently large sample size, the tightness of which will be studied in our future work. In our analysis, we consider the bit-wise noise model with one parameter fixed. Thus, to analyze the running time under general bit-wise noise is also an interesting future work. Note that our analysis has shown that the performance of the (1+1)-EA solving OneMax under bit-wise noise $(1,\frac{\log n}{30n})$ and $(\frac{\log n}{30n},1)$ is significantly different, which implies that $p \cdot q$ may not be the only deciding factor for the analysis of general bit-wise noise.

\begin{acknowledgements}
We want to thank the reviewers for their valuable comments. This work was supported by the NSFC (61603367, 61672478), the YESS (2016QNRC001), the Science and Technology Innovation Committee Foundation of Shenzhen (ZDSYS201703031748284), and the Royal Society Newton Advanced Fellowship (NA150123).
\end{acknowledgements}

\bibliographystyle{spmpsci}      
\bibliography{ectheory}   

\end{document}